\documentclass[12pt,draftclsnofoot,onecolumn]{IEEEtran}
\usepackage{float}
\usepackage[mathscr]{eucal}
\usepackage{ifpdf}
\usepackage{cite}
\usepackage{graphicx}
\usepackage[cmex10]{amsmath}
\usepackage{amssymb}
\usepackage{algorithmic}
\usepackage{units}
\usepackage{setspace}
\usepackage{algorithm}
\usepackage{array}
\usepackage{amsthm}
\usepackage{multirow}
\usepackage{enumerate}
\usepackage{color}
\usepackage{mathtools}
\newcommand{\subparagraph}{}
\usepackage[compact]{titlesec}
\newtheorem{theorem}{Theorem}
\newtheorem{lemma}{Lemma}
\newtheorem{proposition}{Proposition}

\newtheorem{assumption}{Assumption}
\newtheorem{remark}{Remark}

\usepackage{lmodern}
\usepackage{hyperref}
\usepackage{graphics}
\usepackage{bbm}
\DeclarePairedDelimiter\ceil{\Bigg\lceil}{\Bigg\rceil}
\DeclarePairedDelimiter\normalceil{\Big\lceil}{\Big\rceil}
\DeclarePairedDelimiter\floor{\lfloor}{\rfloor}
\usepackage{chngcntr}
\usepackage{tikz}
\usetikzlibrary{arrows,calc}
\DeclareMathOperator*{\argmaxA}{arg\,max}
\DeclareMathOperator*{\dprime}{\prime \prime}

\usepackage{subcaption}
\usepackage{wrapfig,booktabs}

\newcolumntype{R}{>{$}r<{$}}
\newcolumntype{L}{>{$}l<{$}}
\newcolumntype{M}{R@{${}={}$}L}

\makeatletter
\newcommand*{\rom}[1]{\expandafter\@slowromancap\romannumeral #1@}
\makeatother

\newsavebox{\largestimage}
\begin{document}
%-------------------------------------------------------------------> Title
\title{Multi-Armed Bandit for Energy-Efficient and Delay-Sensitive Edge Computing in Dynamic Networks with Uncertainty}
%-------------------------------------------------------------------> Author
\author{
\IEEEauthorblockN{Saeed Ghoorchian and Setareh Maghsudi\\}
\thanks{The authors are with the Electrical Engineering and Computer Science Department, Technical University of Berlin, 10623 Berlin, Germany. (e-mail: saeed.ghoorchian@tu-berlin.de, maghsudi@tu-berlin.de) %\newline A shorter version of this manuscript is currently under submission at IEEE International Conference on Communications.
}
}
\maketitle
%-------------------------------------------------------------------> Abstract
\begin{abstract}
In the edge computing paradigm, mobile devices offload the computational tasks to an edge server by routing the required data over the wireless network. The full potential of edge computing becomes realized only if a smart device selects the most appropriate server in terms of the latency and energy consumption, among many available ones. The server selection problem is challenging due to the randomness of the environment and lack of prior information about the environment. Therefore, a smart device, which sequentially chooses a server under uncertainty, aims to improve its decision based on the historical time and energy consumption. The problem becomes more complicated in a dynamic environment, where key variables might undergo abrupt changes. To deal with the aforementioned problem, we first analyze the required time and energy to data transmission and processing. We then use the analysis to cast the problem as a budget-limited multi-armed bandit problem, where each arm is associated with a reward and cost, with time-variant statistical characteristics. We propose a policy to solve the formulated problem and prove a regret bound. The numerical results demonstrate the superiority of the proposed method compared to a number of existing solutions.
\end{abstract}
%-------------------------------------------------------------------> Keywords
\section{Introduction}
\label{sec:Intro}
The popularity of mobile applications has significantly increased among users over the past years. Some apps, for example, those based on face and/or voice recognition, produce an excessive amount of data and require heavy computations. Even if a hand-held device is capable of performing the computations using its own internal hardware, local data processing usually yield long delay as well as excessive power consumption, thereby resulting in a low Quality of Service (QoS). Moreover, in a long run, repetitive local computation might affect the lifetime of the battery or other components of a mobile device.

In the next-generation wireless networks, edge servers (for example, small base stations) are foreseen to offer computational services, meaning that the devices have the possibility to offload their computational data through a wireless network to the edge servers so that the data is processed remotely. Compared to the cloud servers \cite{Josilo17:AGT}, edge servers are located at close proximity to the users, which guarantees a shorter data transmission time and thereby a lower energy consumption \cite{Abbas18:MEC}, \cite{Arif16:ASO}. Needless to say, edge computing becomes more efficient if the devices are autonomous, i.e., able to choose when and to which server to offload and which resources to use. Implementing an autonomous behavior is, however, not a trivial task. One reason is that unlike cloud servers, there might be multiple edge servers available to the device at the time of offloading. Moreover, often the devices are not given any prior information about the servers and network. In addition, the environment might be dynamic, i.e., some statistical characteristics of the network and servers might change over time.

To deal with the aforementioned challenge, an autonomous device interacts with the network, by sequentially choosing a server under uncertainty, and gathers some information about the environment in each offloading round. The goal is to improve the decisions for the next offloading rounds based on the previously consumed time and energy. This problem is an instance of online decision-making, where the decisions are taken sequentially based on the historical observations to optimize some objective function. 

Multi-Armed Bandit (MAB) problem is a subclass of online decision-making problems which involves a gambling machine with several arms and a gambler \cite{Robbins52:SAS}, \cite{Maghsudi17:MAB}. In this paper, we use an MAB formulation to deal with the optimal server selection problem. We investigate an MAB problem, where pulling each arm reveals two random variables reward and cost. The reward and cost generating processes are a priori unknown and piece-wise stationary. At each of the consecutive rounds, the gambler pulls one arm, receives a reward, and pays a cost. Given a finite budget, the gambler tries to maximize its accumulated reward before the total paid cost runs out of the budget.
%-------------------------------------------------------------------> Subsection Related Works
\subsection{Related Works}
\label{subsec:relatedwork}
Similar to any other networking paradigm, resource management is a key challenge in computation offloading due to the scarcity of resources such as the computational power, environmental and hardware constraints such as the number of available servers, and the dynamic status of the environment such as the task arrival rate. In \cite{Yu17:COF}, the authors take advantage of supervised learning methods to solve an offloading problem in a dynamic environment, where a single user decides which components of the application to execute locally and which ones to offload. They jointly optimize the local execution cost and the offloading cost using a deep neural network. The authors in \cite{Dinh17:OIM} study the CPU task allocation problem by formulating optimization problems based on the execution time and energy consumption. They solve the proposed optimization problems using different approximation approaches. In \cite{Liu18:OSI}, the authors formulate a non-convex optimization problem to optimize both the latency and reliability (offloading failure probability) in computation offloading of a single user. They design three algorithms to optimize edge node candidate selection, offloading ordering, and task allocation. In \cite{Huang12:ADO}, the authors investigate the partial offloading of a single device and propose an algorithm which uses a Lyapunov optimization with a given time delay constraint to reduce the energy consumption. Similarly, \cite{Munoz15:OOR} studies the partial offloading of a single user where multiple antennas are available at the mobile terminal and the femto-access point. The authors propose a numerical optimization technique to optimize latency and energy consumption. Further, in \cite{Wang16:MEC}, the authors consider the partial offloading problem and assume the availability of a small cell cloud manager, which determines whether to offload or not and which portion is needed to be offloaded. They propose different algorithms to separately optimize the latency and consumed energy. \cite{Roy17:AAC} considers the offloading problem of a single user in a multi-cloudlet environment and proposes an application-specific cloudlet selection strategy to optimize the execution latency and energy consumption. In \cite{Van18:ADR}, the authors formulate the partial offloading problem of a single user as a Markov decision process. They use a deep reinforcement learning algorithm to find the optimal number of tasks which should be locally executed or offloaded so that the user's utility is maximized whereas the energy consumption, processing delay, required payment, and task loss probability are minimized. 

As mentioned previously, we model the computation offloading problem in the MAB framework. Our approach is perhaps most closely related to \cite{Ding13:MAB}, where a budgeted MAB problem is considered with a reward and a discrete cost which are independent and identically distributed (i.i.d.) random variables. In \cite{Xia15:TSB}, the authors take a probabilistic approach to solve the budgeted MAB problem with i.i.d. reward and cost variables. Similarly, in \cite{Tran12:KBO}, the authors consider the budgeted MAB problem with i.i.d. reward and cost variables. The proposed algorithm assigns a pulling probability to each arm based on the solution of an optimization problem. Nevertheless, extending the developed decision-making policies to dynamic (non-i.i.d.) environments is not straightforward. Our approach is also related to \cite{Garivier08:SWUCB}, where the authors investigate a non-stationary MAB problem. However, in their formulation, pulling arms does not result in any cost. In \cite{Maghsudi18:PEP}, the authors study a budgeted MAB problem, where the reward generating processes of arms are piece-wise stationary and the cost of pulling each arm is fixed but may be different for different arms. Further, in \cite{Xia15:BBP}, the authors study a stationary MAB problem with a reward variable and a continuous cost variable.
%-------------------------------------------------------------------> Subsection Contribution
\subsection{Our Contribution}
\label{subsec:contribution}
In summary, the novelty in this paper is as follows.
\begin{itemize}
   \item We analyze the required time for data transmission from a user's device to a server, as well as the required time for data processing at a server while taking the dynamic nature and the inhomogeneity of wireless networks into account.
   \item We define the reward and cost in terms of the required time and energy in each offloading round, respectively, and we derive the corresponding probability distributions.
   \item We use an MAB model to solve the distributed server selection problem. Thus, our work extends state-of-the-art works, which are mostly centralized. Moreover, our proposed solution does not require heavy information and does not cause excessive computational complexity.
   \item We propose BPRPC-SWUCB, a novel MAB algorithm, to minimize the expected regret. BPRPC-SWUCB can be used to solve a variety of dynamic decision-making problems where taking actions yields non-i.i.d. reward and cost variables.
   \item We analyze BPRPC-SWUCB by proving a regret bound and compare its performance with several existing MAB algorithms through simulation.
\end{itemize}
%-------------------------------------------------------------------> Subsection Organization
\subsection{Organization}
Section \ref{sec:SysMod} describes the system model. In Section \ref{sec:ProFor}, we introduce the concept of reward and cost in the context of the computation offloading problem, and we derive their statistical characteristics. In Section \ref{sec:MAB}, we describe and theoretically analyze an MAB algorithm, named BPRPC-SWUCB. In Section \ref{sec:NumAna}, we present the results of numerical analysis. Section \ref{sec:Con} concludes the paper.
%-------------------------------------------------------------------> Section System Model
\section{System Model}
\label{sec:SysMod}
We consider a multi-hop wireless network consisting of a set of servers that have fixed locations at the network's edge and a set of users that might be willing to offload their computational job to one of the edge servers. We gather the servers in the set $\mathcal{S}=\{1, \dots, S\}$ so that any device may select one of the $|\mathcal{S}| = S$ to offload its computational task. Throughout the paper, we may use \textit{device} and \textit{offloading user} interchangeably. Moreover, we use the terms, user's device and source, as well as the terms, server and sink, interchangeably.

A general computation offloading procedure consists of four elements: (i) selection of a server, (ii) sending the data to the server, (iii) processing the data and accomplishing the task at the server, and (iv) sending the results to the device. We consider the time to be slotted and denote one time instant by $t$. Moreover, we use the term \textit{round} to refer to the time period required to accomplish an offloading process entirely, i.e., to succeed in all of the aforementioned sections. We denote the rounds by $\theta=1,2,...$. Note that each round $\theta$ includes some time instants $t$.

Each computational job consists of some analysis of the offloaded data. We assume that each computational job can be divided to some homogeneous tasks with respect to the time required to process each task. Without loss of generality, we assume that each device offloads the same amount of the data at each round $\theta$. If a large amount of data is to be offloaded, we model it as multiple rounds of offloading, each with the same amount of data.

As mentioned above, in order to offload a computational task, any user transfers the required data to a server. The transfer takes place via some intermediate helper nodes, which act as transmitters and receivers. This could be, for example, other devices in the network or fixedly deployed micro- or femto small base stations. At each time, every node can act either as a transmitter or as a receiver. We select the transmission range of each node (including the source and any sink $s \in \mathcal{S}$) to be the same and denote it by $R$. That is to say, a node can only transmit to the nodes inside the circle of radius $R$ around that node. In the following, we discuss the network's model from the perspective of one exemplary user.

As it is conventional \cite{Ganti14:DCA}, \cite{Baccelli09:SGA}, we assume that the intermediate nodes (devices, relays, small base stations, and the like), located between the source and a sink, are distributed according to a homogeneous Poisson Point Process (PPP). Since the servers are located at different geographical areas, the density of the aforementioned PPP varies over servers. Therefore, we use $\Lambda_{s}$ to show the network's intensity between the user and each server $s \in \mathcal{S}$. Similar to \cite{Ganti14:DCA} and \cite{Haenggi09:SGA}, to take the transmission impairments of the link between every two nodes into account, we model the links by a Bernoulli random variable with success probability $p_{s,\theta}$. In other words, the transmission is successful (non-outage) with probability $p_{s,\theta}$ and fails (outage) with probability $1-p_{s,\theta}$. Note that the outage probability depends on the server, since the outage probability is affected by factors such as shadowing, fading, and other similar variables which depend strongly on the geographical area as well as the network density. Moreover, the dependency of $p_{s,\theta}$ on the round (time period) $\theta$ accommodates the time-variation of the channel quality. In brief, the network between each server $s$ and the device is modeled by a graph, where the vertices are distributed according to a PPP with intensity $\Lambda_{s}$ and there is an edge between every two vertexes with the probability $p_{s,\theta}$.

As mentioned before, in our problem, we analyze the smart decision-making of a single offloading user, when given a number of choices with respect to the server; nonetheless, it is natural that in every network, there are many of such users, each offloading some tasks to some server. To model the collective behavior of the network mathematically, we assume that the arrived jobs at a server $s$ follow a Poisson distribution with the rate $\lambda_{s,\theta}$. The arrival rate depends on the server $s$ and the offloading round $\theta$, implying that on average, the intensity of the job arrival changes with respect to the servers and time.

In the following assumption, we describe the mathematical model of time-variant characteristics of the random variables.
%-------------------------------------------------------------------> Assumption 1
\begin{assumption} 
\label{Asp1:Nonstationarity} 
For any server $s \in \mathcal{S}$, the parameters $p_{s,\theta}$ and $\lambda_{s,\theta}$ are piece-wise constant with respect to the round $\theta$; in other words, they remain constant unless they experience a change at some specific round(s), referred to as \textit{change point(s)}. Naturally, the change points are not necessarily identical for two aforementioned parameters.

Consider a random process whose instantaneous outcomes are drawn from some probability distribution with parameter $p_{s,\theta}$ and/or $\lambda_{s,\theta}$. Then, by the discussion above, the process is piece-wise stationery, as the distribution of the outcomes remains time-invariant over disjoint time intervals, but changes from one interval to the other.
\end{assumption}
%-------------------------------------------------------------------
%-------------------------------------------------------------------> Figure
\begin{figure}[t]
\begin{center}
% \hspace{0.5\textwidth}
\includegraphics[width=0.7\textwidth]{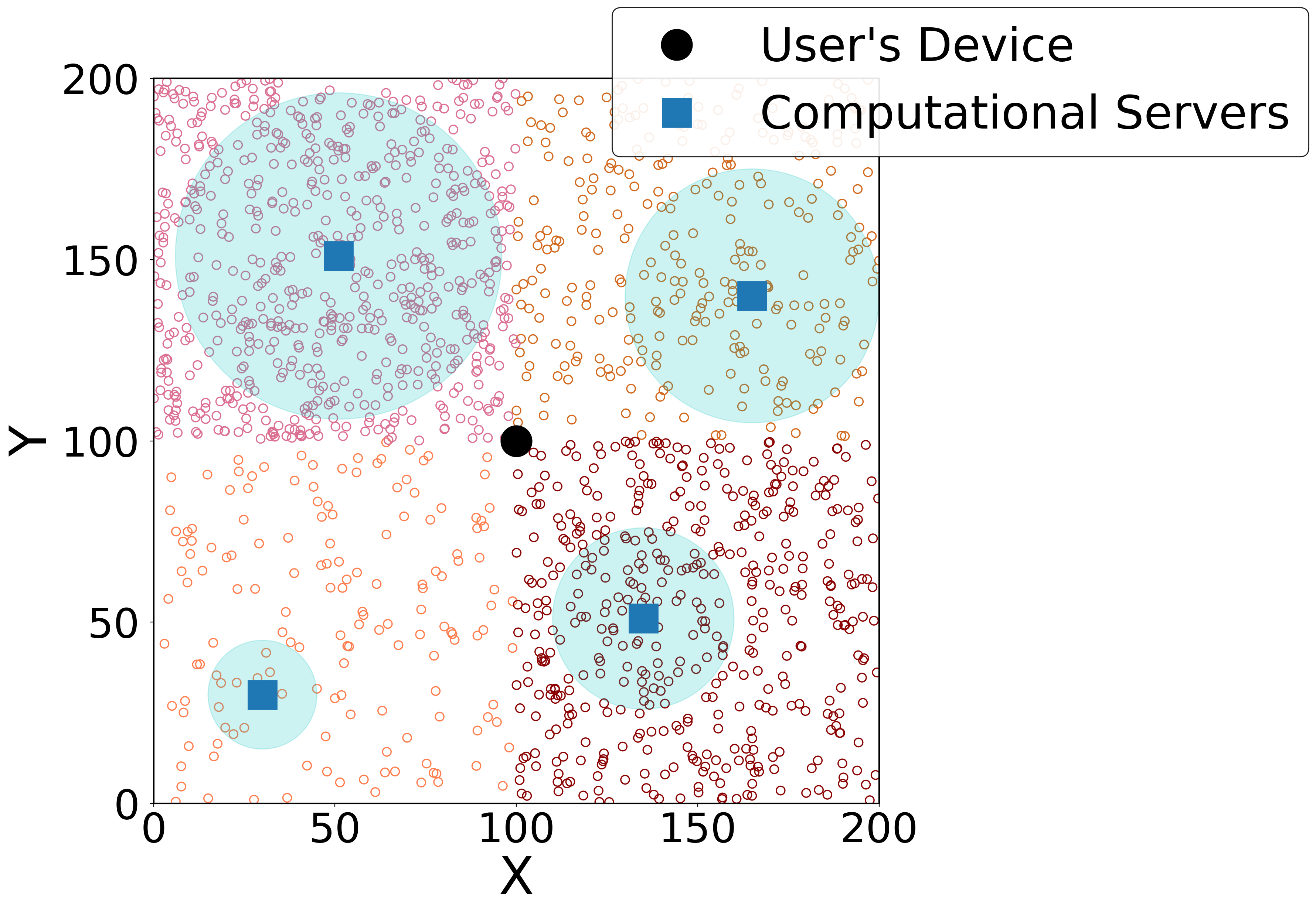}
\end{center}
\caption{An exemplary illustration of a communication network consisting of an offloading user, four computational servers, and the intermediate transmitters and receivers. The intensity of the intermediate nodes varies with respect to each server. The cyan circle around each server represents its job arrival rate, where a bigger radius corresponds to a greater rate.}
\label{NetworkFig}
\end{figure}
%-------------------------------------------------------------------
In \textbf{Fig. \ref{NetworkFig}}, we illustrate an exemplary system model consisting of an offloading user and four edge servers at some specific time $t$. Geographically, the network is divided into four disjoint areas and the nodes in each area are distributed according to a homogeneous PPP. Naturally, in the areas with higher intensity, a larger number of nodes are available. The transparent cyan circle around each server represents the corresponding job arrival rate.

\textbf{Table \ref{Table:Notations}} summarizes most important system's parameters together with a brief description.
%-------------------------------------------------------------------> Table
%Single Column Version
\renewcommand{\arraystretch}{1.1}
\begin{table}[h]
    \begin{center}
    \caption{Summary of most frequently used system parameters}
    \label{Table:Notations}
    \begin{tabular}{c|l}
    \cline{1-1}
    Parameter &   \\ 
    \hline
    $p_{s,\theta}$ & Outage parameter of the network between the user and server $s$ at round $\theta$ \\
    \hline
    $\lambda_{s,\theta}$ & Job arrival rate to the server $s$ at round $\theta$ \\
    \hline
    $\Lambda_{s}$ & Network's intensity between the user and server $s$ \\
    \hline
    $R$ & Transmission range \\
    \hline
    $\rho_{s}$ & Service rate corresponding to the server $s$ \\
    \hline
    $\ell_{s}$ & Distance between the user and the server $s$  \\
    \hline
    % $\delta$ & A threshold as the maximum delay for the user's QoS \\
    % \hline
    % $h_{s,\max}$ & Maximum number of hops between the user's device and a server $s$ \\
    % \hline
    % $p_{g}$ & Energy consumption rate for the data transmission \\
    % \hline
    % $p_{f}$ & Energy consumption rate for the data processing \\
    % \hline
    \end{tabular}
    \end{center}
\end{table}
%
%Double Column Version
% \renewcommand{\arraystretch}{1}
% \begin{table}[h]
% \caption{Summary of system parameters}
% \label{Table:Notations}
% {\footnotesize
% \centering
% \begin{tabular}{c|l}
% \cline{1-1}
%     Parameter &   \\ 
%     \hline
%      \multirow{2}{*}{$p_{s,\theta}$} & Outage parameter of the network \\
%     &  between the user and server $s$ at round $\theta$ \\
%     \hline
%     $\lambda_{s,\theta}$ & Job arrival rate to the server $s$ at round $\theta$ \\
%     \hline
%     $\Lambda_{s}$ & Network's intensity between user and server $s$ \\
%     \hline
%     $\rho_{s}$ & Service rate corresponding to the server $s$ \\
%     \hline
%     $R$ & Transmission range \\
%     \hline
%     $\ell_{s}$ & Distance between the user and server $s$  \\
%     \hline
%     \end{tabular}
% }
% \end{table}
%-------------------------------------------------------------------
%-------------------------------------------------------------------> Section Problem Formulation
\section{Statistical Characteristics of the System Variables}
\label{sec:ProFor}
Conventionally, in wireless networks, each user has some strict constraints (or requirements) on the delay and the energy. Therefore, given multiple choices, it is natural that a device aims at selecting a server that guarantees minimum delay as well as minimum energy consumption. Choosing the best server is however not a trivial task, in particular under uncertainty, i.e., when the required information is not available at the user. The problem becomes more challenging in a dynamic environment, where the characteristics of the network and servers vary over time.

In order to mathematically formulate the server selection problem, in the following, we first define and analyze the reward and cost of selecting each server. 
%-------------------------------------------------------------------> Subsection Reward
\subsection{Reward}
\label{subsec:Reward}
As mentioned earlier, in computation offloading, an important performance metric is the total time required for an offloading round, referred to as the \textit{delay time} and denoted by $d_{s,\theta}$. The delay time at round $\theta$ consists of the processing time $f_{s,\theta}$ at the server $s$ and the transmission time $g_{s,\theta}$ between the source and the sink node $s$. Therefore, at round $\theta$ we have
\begin{equation}
d_{s,\theta} = f_{s,\theta} + g_{s,\theta}.
\end{equation}

For the user's quality of service (QoS) satisfaction, we require that the delay time $d_{s,\theta}$ remains below a pre-specified threshold, namely, $\delta$. In other words, the QoS is satisfied if $d_{s,\theta} \leq \delta$, and is not satisfied otherwise. Therefore, we define the \textit{reward}, gained by the offloading user at round $\theta$ upon choosing the server $s \in \mathcal{S}$, as
\begin{align} 
\label{RewardFormula}
r_{s,\theta}=\begin{cases}
1, & \hspace{3mm} d_{s,\theta} \leq \delta \\
0, & \hspace{3mm} d_{s,\theta}>\delta.
\end{cases}
\end{align}

In the rest of this section, our goal is to find the distribution of the reward $r_{s,\theta}$, which is determined based on the distribution of the delay time $d_{s,\theta}$. Consequently, in the following, we determine the distribution of the processing time $f_{s,\theta}$ and the transmission time $g_{s,\theta}$.
%-------------------------------------------------------------------> Subsubsection Processing Time
\subsubsection{Processing Time}
\label{subsec:Process}
For each server, we define the \textit{service rate} as the number of tasks which can be processed by that server per unit of time. Naturally, the servers are inhomogeneous in terms of service rate, meaning that each server $s \in \mathcal{S}$ has some service rate $\rho_{s}>\lambda_{s,\theta}$, $\forall \theta$. We use $z_{s,\theta}$ to denote the \textit{service time} required by the server $s \in \mathcal{S}$ at round $\theta$.

Moreover, to be processed, each computational job arrived at a server $s \in \mathcal{S}$ has to wait in a queue for some time depending on the job arrival rate. Consider a time instant $t$ inside a round (time period) $\theta$. We denote the \textit{waiting time} at time instant $t$ by $w_{s,t}$. Similarly, we use $f_{s,t}$ and $z_{s,t}$ to denote the processing time and the service time at time instant $t$, respectively. Thus, at server $s \in \mathcal{S}$, the processing time at time $t$ is given by
\begin{align}
f_{s,t}=z_{s,t}+w_{s,t}.
\end{align}

We consider an $M/M/1$ queue model, by which $z_{s,t}$ and $f_{s,t}$ follow an exponential distribution with parameter $\rho_{s}$ and $\rho_{s}-\lambda_{s,t}$, respectively  \cite{Sztrik12:BQT}, \cite{Ranadheera18:COA}. By Assumption \ref{Asp1:Nonstationarity}, the job arrival rate remains fixed at least during a specific round $\theta$. Therefore, for any time instant $t$ inside a round $\theta$, it holds $\lambda_{s,\theta}=\lambda_{s,t}$. In words, this implies that the expected value of the waiting time, and consequently of the processing time, remains constant for the entire time period of an offloading round $\theta$. Therefore, throughout the paper, we use $f_{s,\theta}$ to denote the \textit{processing time} at the server $s$ for round $\theta$, regardless of the specific time instant $t$ inside the round $\theta$. Moreover, note that by Assumption \ref{Asp1:Nonstationarity}, $\lambda_{s,\theta}$ is assumed to be piece-wise constant, which implies that $f_{s,\theta}$ follows an exponential distribution with piece-wise constant mean $\frac{1}{\rho_{s}-\lambda_{s,\theta}}$. Formally,
%
%Single Column Version
\begin{equation} \label{ProbProcessTime}
    \mathbb{P}(f_{s,\theta} = x) 
    = \begin{cases} 
    (\rho_{s} - \lambda_{s,\theta}) e^{-(\rho_{s} - \lambda_{s,\theta})x}, \hspace{10mm} x \geq 0\\
    0, \hspace{45.5mm} x < 0
    \end{cases}
\end{equation}
%
% Double Column Version
% \begin{equation} \label{ProbProcessTime}
%     \mathbb{P}(f_{s,\theta} = x) 
%     = \begin{cases} 
%     (\rho_{s} - \lambda_{s,\theta}) e^{-(\rho_{s} - \lambda_{s,\theta})x}, \hspace{6mm} x \geq 0\\
%     0, \hspace{38.4mm} x < 0
%     \end{cases}
% \end{equation}
%
%-------------------------------------------------------------------> Subsubsection Transmission Time
\subsubsection{Transmission Time}
\label{subsec:Transmission}
A path of length $N$ is an $N$-hop connection between the source $o$ and the sink $s$. We represent such path by a sequence $o = x_{1}, x_{2}, \dots, x_{N+1} = s$, where $x_{i}$ denotes the $i$-th node in the path and $x_{1}$ and $x_{N+1}$ stand for the source and the sink, respectively. Similar to \cite{Takagi84:OTR} and \cite{Contla03:EHC}, we define the concept of \textit{progress}. Assume a transmitter node located at $x_{i}$. The progress of a node $x_{i+1}$ is defined as the projection of the link between $x_{i}$ and $x_{i+1}$ onto the straight line connecting the node $x_{i}$ and the sink $s$. Additionally, we say a progress is positive if the projection happens towards the sink $s$ and it is negative otherwise. We define the maximum number of hops $h_{s,\max}$ between the source $o$ and the sink $s$ as the maximum $N$ for which a path exists between $o$ and $s$ and all the nodes $x_{i}$, $i=2, \dots, N+1$ have positive progress. We assume that $h_{s,\max}$ between the source $o$ and any sink $s$ is known.

%-------------------------------------------------------------------> Figure
\begin{figure}[t]
\begin{center}
% \hspace{0.5\textwidth}
\includegraphics[width=0.7\textwidth]{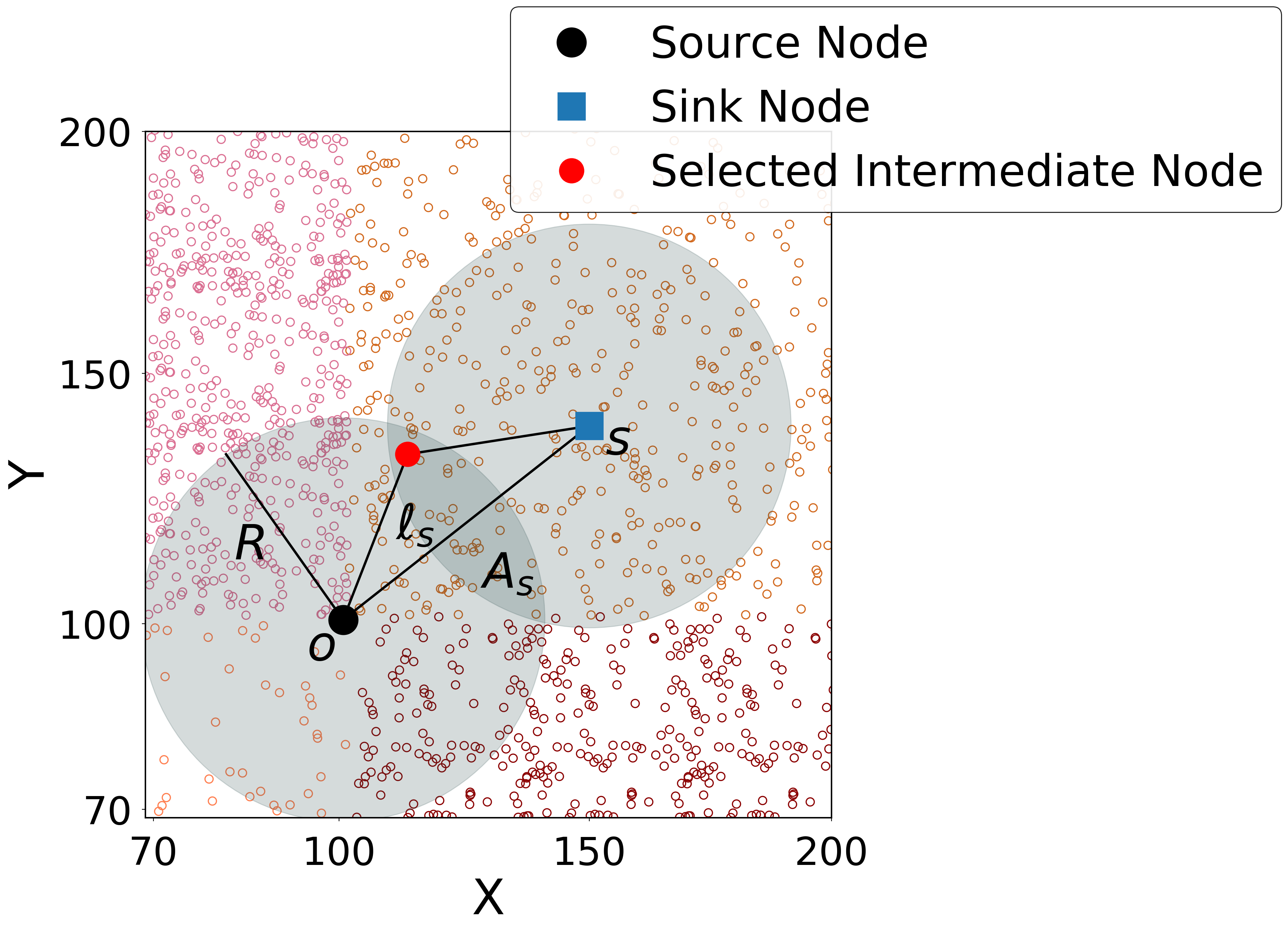}
\end{center}
\caption{Sketch of a 2-hop communication path between the user (source) and a computation server (sink).}
\label{ProtocolFig}
\end{figure}
%-------------------------------------------------------------------
In \textbf{Fig. \ref{ProtocolFig}}, a source node $o$ transmits a data packet to the sink $s$. For the pair $(o,s)$, we define the \textit{distance} as the length of the straight line connecting the source $o$ and the sink $s$. According to our system model, the distance is known, which we denote by $\ell_{s}$. If the sink $s$ is not located within the transmission range of the source $o$, the data should be transmitted using the intermediate nodes of the PPP. Therefore, several hops might be needed to transmit the data from the source to the sink. Let $H_{s}$ denote the random number of hops between the source and a sink $s$. The probability of connecting the source $o$ and the sink $s$ with $h$ number of hops, $h= 1,2,...$, is computed in \cite{Harb08:ASO} as
\begin{align} 
\label{ProbHop}
\mathbb{P}(H_{s} = h) &= C_{\ell_{s}}[1 - e^{-\Lambda_{s}|A_{s}|}]^{h-1},
\end{align}
where $C_{\ell_{s}}$ is a constant which depends on the distance $\ell_{s}$ between the source and the sink node $s$ and $0 \leq C_{\ell_{s}} \leq 1$. Moreover, in (\ref{ProbHop}), $A_{s}$ denotes the intersection area between the transmission range of a node and its next node in a path which can be calculated as \cite{Harb08:ASO}
\begin{equation}
|A_{s}| = R^{2}[2\cos^{-1}(\frac{\ell_{s}}{2R})] - \sin(2\cos^{-1}(\frac{\ell_{s}}{2R}))].
\end{equation}
Hence, the expected value of the number of hops $H_{s}$ yields
%
%Single Column Version
\begin{align} 
\label{ExpHops}
\mathbb{E}[H_{s}] = \sum_{H_{s} = 1}^{h_{s,\max}} H_{s} \mathbb{P}(H_{s}) = C_{\ell_{s}} \sum_{H_{s} = 1}^{h_{s,\max}} H_{s}[1 - e^{-\Lambda_{s}|A_{s}|}]^{(H_{s}-1)}.
\end{align}
%
%Double Column Version
% \begin{align} 
% \label{ExpHops}
% \mathbb{E}[H_{s}] = \hspace{-1mm} \sum_{H_{s} = 1}^{h_{s,\max}} \hspace{-1mm} H_{s} \mathbb{P}(H_{s}) = C_{\ell_{s}} \hspace{-1mm} \sum_{H_{s} = 1}^{h_{s,\max}} \hspace{-1mm} H_{s}[1 - e^{-\Lambda_{s}|A_{s}|}]^{(H_{s}-1)}.
% \end{align}
%

However, in our setting, there is a possibility of outage for a transmission between any pair of nodes $x_{i}$ and $x_{i+1}$; this means that the transmitter might require several attempts until a successful reception at the receiver is achieved. Let $K_{i}$, $i=1, 2, \dots$, denote the random variable representing the number of Bernoulli trials (time instants) needed for the first successful connection between the transmitter-receiver pair $x_{i}$ and $x_{i+1}$. Then we have
\begin{equation} 
\label{ProbGeometric}
\mathbb{P}(K_{i}=k_{i})= p_{s,\theta}(1-p_{s,\theta})^{k_{i}-1}.
\end{equation}
In words, the number of time instants (attempts) needed to achieve the first successful connection follows a geometric distribution.

The \textit{transmission time} $g_{s,\theta}$ between the source $o$ and the sink $s$ at round $\theta$ is given by
\begin{equation} 
\label{TimeTransmissionDefinition}
g_{s,\theta} = \sum_{i=1}^{H_{s}} K_{i}.
\end{equation}
The following proposition states the statistical characteristics of the transmission time.
%-------------------------------------------------------------------> Proposition 1
\begin{proposition} 
\label{Prop1}
The transmission time $g_{s,\theta}$ is a random variable with the probability distribution
%Single Column Version
\begin{equation} 
\label{ProbTime}
\mathbb{P}(g_{s,\theta} = k) = C_{\ell_{s}} \sum_{h=1}^{\min{\{k,h_{s,\max}\}}} \binom{k - 1}{h - 1} p_{s,\theta}^{h}(1 - p_{s,\theta})^{k-h} [1 - e^{-\Lambda_{s}|A_{s}|}]^{h-1}, \hspace{10mm} k=1, 2, \dots,
\end{equation}
%
%Double Column Version
% \begin{align} 
% \label{ProbTime} \nonumber
% \mathbb{P}(g_{s,\theta} = k) = ~&C_{\ell_{s}} \sum_{h=1}^{\min{\{k,h_{s,\max}\}}} \binom{k - 1}{h - 1} \times \\
% &\times p_{s,\theta}^{h}(1 - p_{s,\theta})^{k-h} [1 - e^{-\Lambda_{s}|A_{s}|}]^{h-1},
% \end{align}
%
% for $k=1, 2, \dots$,
%
and the expected value
\begin{equation} 
\label{ExpTime}
\mathbb{E}[g_{s,\theta}] = \frac{C_{\ell_{s}} \sum_{H_{s} = 1}^{h_{s,\max}} H_{s}[1 - e^{-\Lambda_{s}|A_{s}|}]^{(H_{s}-1)}}{p_{s,\theta}}.
\end{equation}
\end{proposition}
%-------------------------------------------------------------------> Proof
\begin{proof}
See Appendix \ref{ProOneProof}.
\end{proof}
%-------------------------------------------------------------------
We observe that the expected transmission time depends on the outage parameter $p_{s,\theta}$; therefore, in a dynamic environment where $p_{s,\theta}$ is piece-wise constant, $g_{s,\theta}$ has a piece-wise constant mean, as given by (\ref{ExpTime}).
%-------------------------------------------------------------------> Subsubsection Delay Time
\subsubsection{Delay Time and Reward}
\label{subsec:DelayTime}
Finally, the following proposition characterizes the statistics of the variable reward.
%-------------------------------------------------------------------> Proposition 2
\begin{proposition} 
\label{Prop2}
Reward $r_{s,\theta}$ is a random variable with Bernoulli distribution. Moreover, it has a piece-wise constant expected value as 
%Single Column Version
\begin{equation}
\mu_{s,\theta} = C_{\ell_{s}} \sum_{k=1}^{\floor{\delta}} {\Big[} {\Big(} 1 - e^{-(\rho_{s} - \lambda_{s,\theta})(\delta - k)} {\Big)} \sum_{h=1}^{\min{\{k,h_{s,\max}\}}} \binom{k - 1}{h - 1} p_{s,\theta}^{h}(1 - p_{s,\theta})^{k-h} [1 - e^{-\Lambda_{s}|A_{s}|}]^{h-1} {\Big]}.
\end{equation}
%
%Double Column Version
% \begin{align} \nonumber
% \mu_{s,\theta} \hspace{-1mm}&= C_{\ell_{s}} \sum_{k=1}^{\floor{\delta}} {\Big[} {\Big(} 1 - e^{-(\rho_{s} - \lambda_{s,\theta})(\delta - k)} {\Big)} \hspace{-2.5mm} \sum_{h=1}^{\min{\{k,h_{s,\max}\}}} \hspace{-2mm} \binom{k - 1}{h - 1} \times \\ 
% &\times p_{s,\theta}^{h}(1 - p_{s,\theta})^{k-h} [1 - e^{-\Lambda_{s}|A_{s}|}]^{h-1} {\Big]}.
% \end{align}
%
\end{proposition}
%-------------------------------------------------------------------> Proof
\begin{proof}
See Appendix \ref{ProTwoProof}.
\end{proof}
%-------------------------------------------------------------------
%-------------------------------------------------------------------> Subsection Cost
\subsection{Cost}
\label{subsec:Cost}
Naturally, every offloading round results in some energy consumption due to data transmission to the server and data processing at the server. Consider a round $\theta$ in which the computational task is offloaded to a server $s$. We denote the total required energy by $c_{s,\theta}$. Due to the energy scarcity, we define the \textit{cost} in terms of the consumed energy. In general, the consumed energy, i.e., the cost $c_{s,\theta}$, is a function of the transmission time and processing time. More precisely, it consists of the following parts.
\begin{itemize}
\item The energy required for data transmission, denoted by $v_{g}(g_{s,\theta})p_{g}$, where $p_{g}$ is the energy consumption rate for data transmission. Note that $g_{s,\theta}$ represents the time required for sending the data from the user to the server $s$ at round $\theta$. However, we need to consider the time required for sending the data from the server back to the user at the same round $\theta$. We consider that the function $v_{g}(\cdot)$ takes into account this round trip, for instance, via additionally multiplying $g_{s,\theta}$ by $2$.
\item The energy required for data processing at the server, denoted by $v_{f}(f_{s,\theta})p_{f}$, where $p_{f}$ is the energy consumption rate for accomplishing the job.
\end{itemize}

Note that $p_{g}$ and $p_{f}$ are known system parameters. Generally, $v_{g}(\cdot)$ and $v_{f}(\cdot)$, can be any invertible function; in this paper, for the sake of computation, we consider linear functions. Consequently, we have 
\begin{equation} 
\label{cost}
c_{s,\theta} = a_{s} f_{s,\theta} + a^{\prime}_{s} g_{s,\theta} + a^{\dprime}_{s},  
\end{equation}
where $a_{s}$, $a^{\prime}_{s} > 0$, and $a^{\dprime}_{s} \geq 0$. Hence, $\min\limits_{s,\theta} c_{s,\theta} = a^{\prime}_{s} + a^{\dprime}_{s}$. Note that the cost $c_{s,\theta}$ takes its minimum when the data is successfully transmitted via only one hop and in the first attempt and also when the processing time $f_{s,\theta} = 0$.

The following proposition determines the statistical characteristics of the variable cost.
%-------------------------------------------------------------------> Proposition 3
\begin{proposition} \label{Prop3}
The cost $c_{s,\theta} \geq a^{\prime}_{s} + a^{\dprime}_{s}$ for an offloading round $\theta$ between the user's device and any server $s$ is a random variable with the probability distribution as follows
%Single Column Version
\begin{align} \label{CostDist} \nonumber
\mathbb{P}(c_{s,\theta}=x) = \frac{C_{\ell_{s}}}{a_{s}} \sum_{k=1}^{\floor{\frac{x-a^{\dprime}_{s}}{a^{\prime}_{s}}}}{\Big[}&{\Big(}(\rho_{s} - \lambda_{s,\theta}) e^{-(\rho_{s} - \lambda_{s,\theta})(\frac{x-a^{\dprime}_{s}-a^{\prime}_{s}k}{a_{s}})}{\Big)} \times \\
& \times \sum_{h=1}^{\min{\{k,h_{s,\max}\}}} \binom{k - 1}{h - 1}p_{s,\theta}^{h}(1-p_{s,\theta})^{k-h}[1- e^{-\Lambda_{s}|A_{s}|}]^{h-1}{\Big]}.
\end{align}
%Double Column Version
% \begin{align} \label{CostDist} \nonumber
% &\mathbb{P}(c_{s,\theta} \hspace{-0.6mm} = \hspace{-0.6mm} x) \hspace{-0.6mm}=\hspace{-0.6mm} \frac{C_{\ell_{s}}}{a_{s}} \hspace{-2mm}\sum_{k=1}^{\floor{\frac{x-a^{\dprime}_{s}}{a^{\prime}_{s}}}} \hspace{-2mm} {\Big[}{\Big(}(\rho_{s} - \lambda_{s,\theta}) e^{-(\rho_{s} - \lambda_{s,\theta})(\frac{x-a^{\dprime}_{s}-a^{\prime}_{s}k}{a_{s}})}{\Big)} \times \\ 
% &\times \hspace{-2.5mm} \sum_{h=1}^{\min{\{k,h_{s,\max}\}}} \hspace{-.51mm} \binom{k - 1}{h - 1}p_{s,\theta}^{h}(1-p_{s,\theta})^{k-h}[1- e^{-\Lambda_{s}|A_{s}|}]^{h-1}{\Big]}.
% \end{align}
%
Moreover, its expected value is equal to
%Single Column Version
\begin{equation} 
\label{ExpCost}
\eta_{s,\theta} = \frac{a_{s}}{\rho_{s} - \lambda_{s,\theta}}  + \frac{a^{\prime}_{s} C_{\ell_{s}}\sum_{H_{s}=1}^{h_{s,\max}} H_{s}[1-e^{-\Lambda_{s}|A_{s}|}]^{(H_{s}-1)}}{p_{s,\theta}} + a^{\dprime}_{s}.
\end{equation}
%Double Column Version
% \begin{align} 
% \label{ExpCost} \nonumber
% &\eta_{s,\theta} = \\
% &\frac{a_{s}}{\rho_{s} - \lambda_{s,\theta}}  + \frac{a^{\prime}_{s} C_{\ell_{s}}\sum_{H_{s}=1}^{h_{s,\max}} H_{s}[1-e^{-\Lambda_{s}|A_{s}|}]^{(H_{s}-1)}}{p_{s,\theta}} + a^{\dprime}_{s}.
% \end{align}
%
\end{proposition}
%-------------------------------------------------------------------> Proof
\begin{proof}
See Appendix \ref{ProThreeProof}.
\end{proof}
%-------------------------------------------------------------------
The user devices play a crucial role in multi-hop wireless networks. Such devices consume the energy stored in their batteries to participate in the process of computation offloading, necessitating a frequent recharge. Moreover, the energy resources of mobile devices and edge servers are often unsustainable and not environment-friendly. Consequently, we consider a limit for the energy spent during the computation offloading. We refer to this limit as the \textit{budget} and denote it by $B$. Naturally, $B$ is a deterministic constant and known to the user. Therefore, the offloading user continues to offload the computational jobs as long as the total spent energy, i.e., the total paid cost, does not exceed the budget $B$.
%-------------------------------------------------------------------> Section Solution 
\section{Model and Solution based on Multi-Armed Bandits}
\label{sec:MAB}
To solve the server selection problem, we take advantage of a class of sequential optimization problems with limited information, namely, the Multi-Armed Bandit (MAB) problem \cite{Maghsudi17:MAB}. In this section, we formulate the server selection problem in the MAB framework and propose an algorithm to solve this problem.
%-------------------------------------------------------------------> Subsection BPRPC-SWUCB
\subsection{Budget-Limited Multi-Armed Bandits with Piece-wise Stationary Reward and Cost}
We consider an MAB problem which portraits a player (device) facing a number of arms (servers). We denote the set of arms of the MAB by $\mathcal{S}=\{1, 2, \dots, S\}$.  By pulling an arm $i \in \mathcal{S}$ in each round $\theta=1,2,...$, the player pays some cost $c_{i,\theta}$ and receives some reward $r_{i,\theta}$. We assume that the random process of reward and cost are unknown a priori and piece-wise stationary. Reward and cost of each arm $i \in \mathcal{S}$ follow a probability distribution with mean $\mu_{i,\theta}$ and $\eta_{i,\theta}$ at round $\theta$, respectively. The rewards are upper bounded, i.e., there is a constant $r_{\max}>0$ such that $0 \leq r_{i,\theta} \leq r_{\max}$,~$\forall i,\theta$. The costs are lower bounded, i.e., there is a constant $0 < c_{min}$ such that $c_{\min} \leq c_{i,\theta}$~$\forall i,\theta$. The player can continue gambling as long as its cumulative cost remains below a given budget $B$. 
Ideally, the player's goal is to maximize its expected accumulated reward until the last round, which we refer to as the \textit{stopping round}. We denote by $T^{\ast}(B)$ and $T(B)$ the stopping round of the optimal policy (known as Oracle) and the stopping round of the applied policy, respectively. Formally, the problem can be formulated as
%Single Column Version
\begin{align} \label{OptPro1} \nonumber
\underset{I_{\theta} \in \mathcal{S}}{\textup{maximize}}~~&\mathbbm{E}{\Bigg [} \sum_{\theta=1}^{T(B)} r_{I_{\theta},\theta} {\Bigg ]} 
\\
\text{s.t.} \hspace{5mm} &\sum_{\theta=1}^{T(B)} c_{I_{\theta},\theta} \leq B,
\end{align}
%Double Column Version
% \begin{align} \label{OptPro1}
% &\underset{I_{\theta} \in \mathcal{S}}{\textup{maximize}}~~\mathbbm{E}{\Bigg [} \sum_{\theta=1}^{T(B)} r_{I_{\theta},\theta} {\Bigg ]} \hspace{5mm} \text{s.t.} \hspace{3mm} \sum_{\theta=1}^{T(B)} c_{I_{\theta},\theta} \leq B,
% \end{align}
%
where $I_{\theta}$ denotes the played arm at round $\theta$.

The Problem (\ref{OptPro1}) is infeasible to solve since the instantaneous outcome of the random variables reward and cost are not known a priori. Moreover, $T(B)$ is a random variable because it depends on the summation of some random variable cost, which by itself depends on the choice of the arm. Therefore, we suggest an alternative problem formulation, as described in the following. First, we define the utility in a way that it includes both reward and cost revealed by an arm upon pulling. Such utility can be used to evaluate the efficiency of a choice of arm as it takes both the reward and cost into account. More precisely, we define the \textit{utility} as reward per cost. Formally,
\begin{align}
\label{eq:Utility}
u_{I_{\theta},\theta}=\frac{r_{I_{\theta},\theta}}{c_{I_{\theta},\theta}}.
\end{align}
We then define the \textit{regret} as the difference between the accumulated reward of Oracle and the accumulated reward of the player under the applied policy. Formally,
\begin{equation}
R_{T(B)} = \sum_{\theta = 1}^{T^{\ast}(B)} r_{i^{\ast}_{\theta},\theta} -
\sum_{\theta = 1}^{T(B)} r_{I_{\theta},\theta},
\end{equation}
where $i^{\ast}_{\theta} = \argmaxA\limits_{i \in \mathcal{S}} \frac{\mu_{i,\theta}}{\eta_{i,\theta}}$ is the arm chosen by Oracle at round $\theta$. Then the player's goal is to minimize the expected regret, i.e.,
\begin{align} 
\label{OptPro2}
\underset{I_{\theta} \in \mathcal{S}}{\textup{minimize}}~~\mathbb{E}[R_{T(B)}].
\end{align}

We propose \textbf{Algorithm \ref{Alg1}} to solve the Problem (\ref{OptPro2}). In this algorithm, we define the average reward and cost as 
\begin{align}
\label{eq:raverage}
\bar{r}_{\theta}(\tau,i) = \frac{\sum\limits_{k=\max\{1, \theta-\tau+1\}}^{\theta} r_{i,k} \mathbbm{1}_{\{I_k = i\}}}{N_{\theta}(\tau,i)}, 
\end{align}
and 
\begin{align}
\label{eq:caverage}
\bar{c}_{\theta}(\tau,i) = \frac{\sum\limits_{k=\max\{1, \theta-\tau+1\}}^{\theta} c_{i,k}\mathbbm{1}_{\{I_k= i\}}}{N_{\theta}(\tau,i)},
\end{align}
respectively, where 
%Single Column Version
% \begin{align*}
% N_{\theta}(\tau,i) = \sum_{k=\theta-\tau+1}^{\theta} \mathbbm{1}_{\{I_k = i\}}.
% \end{align*}
%
% We also define
%Double Column Version
$N_{\theta}(\tau,i) = \sum\limits_{ k=\max\{1, \theta-\tau+1\}}^{\theta} \mathbbm{1}_{\{I_k = i\}}$. We also define
\begin{align}
\label{eq:padding}
E_{\theta}(\tau,i) = \frac{(1 + \frac{r_{\max}}{c_{\min}}) r_{\max} \sqrt{\frac{\xi \log{(min\{\theta,\tau\})}}{N_{\theta}(\tau,i)}}}{c_{\min} - r_{\max} \sqrt{\frac{\xi \log{(min\{\theta,\tau\})}}{N_{\theta}(\tau,i)}}},
\end{align}
where $\xi$ and $\tau$ are tunable parameters. We will elaborate on the choice of these parameters later in Section \ref{sec:NumAna}.

In the initialization phase, Algorithm \ref{Alg1} solely explores the set of arms by selecting each arm once and observing its reward and cost. It then uses the observations to develop an initial approximation for the Upper Confidence Bound (UCB) on the reward-to-cost ratio for each arm. Afterward, the algorithm continues selecting arms until the accumulated cost exceeds the budget $B$. In this stage, at each round $\theta$, the algorithm first calculates the UCB index $\mathcal{I}_{i,\theta}$ for each arm $i \in \mathcal{S}$ and then selects the arm with the highest index. 

As a comparison to other budgeted MAB algorithms, such as KUBE \cite{Tran12:KBO} and UCB-BV1 \cite{Ding13:MAB}, BPRPC-SWUCB is able to detect the changes in the mean reward or mean cost faster and thereby comply faster with the abrupt changes in the environment. This is due to the fact that BPRPC-SWUCB uses a window length $\tau$ and takes only the last $\tau$ observations to calculate the UCB index for each arm.
%-------------------------------------------------------------------> Algorithm - New version
\begin{algorithm}[ht]
\caption{BPRPC-SWUCB: Budget-limited Piece-wise stationary Reward with Piece-wise stationary Cost-Sliding Window Upper Confidence Bound}
\label{Alg1}
\begin{algorithmic}[1]
\STATE \textbf{Input}: Window length $\tau$, parameters $\xi$, $r_{\max}$, and $c_{\min}$
\hspace{-2mm}
\FOR{$\theta = 1, \dots, S$}
\STATE Select arm $I_\theta=\theta$.
\STATE Observe the reward $r_{I_{\theta},\theta}$ and the cost $c_{I_{\theta},\theta}$.
\ENDFOR
\WHILE{$\sum_{k=1}^{\theta} c_{I_{k},k} \leq B$}
\STATE Calculate the index of each arm $i \in \mathcal{S}$ as
\begin{align}
\label{eq:index}
\mathcal{I}_{i,\theta} = \frac{\bar{r}_{\theta}(\tau,i)}{\bar{c}_{\theta}(\tau,i)} + E_{\theta}(\tau,i),
\end{align}
where $\bar{r}_{\theta}(\tau,i)$ and $\bar{c}_{\theta}(\tau,i)$ are defined in (\ref{eq:raverage}) and (\ref{eq:caverage}), respectively. Moreover, $E_{\theta}(\tau,i)$ is defined in (\ref{eq:padding}).
\STATE Select the arm $I_{\theta}$ with the highest index. Formally,
\begin{align}
\label{eq:selectedarm}
I_\theta = \text{arg\,max}_{i \in \mathcal{S}} \,\,\, \mathcal{I}_{i,\theta}.
\end{align}
\vspace{-5mm}
\STATE Observe the reward $r_{I_{\theta},\theta}$ and the cost $c_{I_{\theta},\theta}$.
\STATE Set $\theta = \theta + 1$.
\ENDWHILE
\end{algorithmic}
\end{algorithm}
%-------------------------------------------------------------------
%-------------------------------------------------------------------> Section Analysis of BPRPC-SWUCB
\section{Analysis of BPRPC-SWUCB}
In this section, we prove an upper bound on the expected regret of the BPRPC-SWUCB. We use the following definition in the rest of this paper.
\begin{align} \label{delta}
\Delta(i) = \min{\Bigg\{}\frac{\mu_{i^{\ast}_{\theta},\theta}}{\eta_{i^{\ast}_{\theta},\theta}} - \frac{\mu_{i,\theta}}{\eta_{i,\theta}} \hspace{1mm} {\Bigg|} \hspace{1mm} \forall \theta \in \{1, \dots, T(B)\}~\text{s.t.}~i \neq i^{\ast}_{\theta} {\Bigg\}}.
\end{align}

We first prove an upper bound on the expected cumulative reward of the optimal policy in the following lemma.
%-------------------------------------------------------------------> Lemma 1
\begin{lemma}
\label{Lemma1}
The solution of Problem (\ref{OptPro1}) is upper bounded by $\frac{(B+c_{\min})r_{\max}}{c_{\min}}$. 
\end{lemma}
%-------------------------------------------------------------------> Proof
\begin{proof}
See Appendix \ref{LemmaOneProof}.
\end{proof}
%-------------------------------------------------------------------
In the next theorem, we establish an upper bound on the expected regret of BPRPC-SWUCB.
%-------------------------------------------------------------------> Theorem 1
\begin{theorem} \label{Thm1}
Let us denote by $\Upsilon_{T(B)}$ the number of change points before the stopping round $T(B)$ corresponding to both the reward and cost distribution. If there exists $c_{\max} > 0$ such that $c_{i,\theta} \leq c_{\max}$ $\forall i,\theta$, then for $\xi > \frac{1}{2}$ and any integer $\tau$ we have
%Single Column Version
\begin{align} \label{result1}
    \mathbbm{E}[R_{T(B)}] \leq r_{\max} {\Bigg(} {\Big(} \frac{B}{c_{\min}}(1 - \frac{c_{\min}}{c_{\max}}) + 1 {\Big)} + \sum_{i=1}^{S}  {\Big(} C(\tau,i) \frac{B}{c_{\min}} \frac{\log{(\tau)}}{\tau} + \tau \Upsilon_{T(B)} + 2\log^{2}(\tau) {\Big)} {\Bigg)},
\end{align}
%
%Double Column Version
% \begin{align} \label{result1} \nonumber
%     &\mathbbm{E}[R_{T(B)}] \leq r_{\max} {\Bigg(} 1 + \\
%     &+ \sum_{i=1}^{S}  {\Big(} C(\tau,i) \frac{B}{c_{\min}} \frac{\log{(\tau)}}{\tau} + \tau \Upsilon_{T(B)} + 2\log^{2}(\tau) {\Big)} {\Bigg)},
% \end{align}
%
where
%Single Column Version
\begin{align} \label{Cofsuboptimal1}
    C(\tau,i) = {\Bigg(}\frac{2(1 + \frac{r_{\max}}{c_{\min}}) + \Delta(i)}{c_{\min} \Delta(i)}{\Bigg)}^{2} r_{\max}^{2}\xi \frac{\normalceil{\frac{B}{c_{\min}\tau}}}{\frac{B}{c_{\min}\tau}} + \frac{4}{\log{(\tau)}} \ceil{\frac{\log{(\tau)}}{\log{(1 + 4 \sqrt{1 - (2 \xi)^{-1}})}}}.
\end{align}
%
%Double Column Version
% \begin{align} \label{Cofsuboptimal1} \nonumber
%     C(\tau,i) &= {\Bigg(}\frac{2(1 + \frac{r_{\max}}{c_{\min}}) + \Delta(i)}{c_{\min} \Delta(i)}{\Bigg)}^{2} r_{\max}^{2}\xi \frac{\normalceil{\frac{B}{c_{\min}\tau}}}{\frac{B}{c_{\min}\tau}} \\
%     &+ \frac{4}{\log{(\tau)}} \ceil{\frac{\log{(\tau)}}{\log{(1 + 4 \sqrt{1 - (2 \xi)^{-1}})}}}.
% \end{align}
%
\end{theorem}
%-------------------------------------------------------------------> Proof
\begin{proof}
See Appendix \ref{TheoremOneProof}.
\end{proof}
%-------------------------------------------------------------------
%-------------------------------------------------------------------> Remark 1
\begin{remark} \label{remardfortau}
Based on the distribution of cost in (\ref{CostDist}), we observe that $\forall i,\theta$, $\mathbbm{P}(c_{i,\theta} = c_{\max}) \rightarrow 0$ when $c_{\max} \rightarrow \infty$. Nevertheless, the regret bound (\ref{result1}) in Theorem \ref{Thm1} also holds true when the cost variable is unbounded from above, which is the case in the computation offloading problem discussed in this paper. In this case, when $c_{\max} \rightarrow \infty$, we achieve a regret bound of order $O(B)$.

If $\frac{c_{\min}}{c_{\max}} \rightarrow 1$ (which is the case in many problems, for example, in the problems where the cost is fixed for all arms or it is supported in a small interval), the first term in the regret bound tends to zero. In this case, by choosing $\tau = \sqrt{\frac{B \log(B)}{\Upsilon_{T(B)}}}$, and if we assume that the growth rate of the number of change points $\Upsilon_{T(B)}$ is $O(B^{\alpha})$, for some $\alpha \in [0,1)$, we achieve a regret bound of order $O{\Big(} {B^{\frac{(1+\alpha)}{2}} \sqrt{\log(B)}} {\Big)}$.
\end{remark}
%-------------------------------------------------------------------> Remark 2
% \textcolor{blue}{
% \begin{remark}
% For an MAB setting with stationary random processes of reward and cost, we have $\Upsilon_{T(B)} = 0$. In this case, there is no need to detect the changes in the environment and the average reward (\ref{eq:raverage}) and the average cost (\ref{eq:caverage}) can be calculated using the entire history. Hence, in this case, we can choose $\tau = T(B)$. Replacing $\Upsilon_{T(B)} = 0$ and $\tau = T(B)$ in (\ref{result1}) yields a regret bound of order $O(B)$. Additionally, if $\frac{c_{\min}}{c_{\max}} \rightarrow 1$, the bound will be equal to $O(\log(B))$. This shows that BPRPC-SWUCB is also capable of solving MAB problems with stationary reward and cost variables.
% \end{remark}
% }
%-------------------------------------------------------------------> Remark 3
\begin{remark}{Computational Complexity} \\
The computational complexity of BPRPC-SWUCB is linear with respect to the stopping round T(B). Note that BPRPC-SWUCB only stores the action and reward/cost history of the last $\tau$ rounds, hence it is more space-efficient compared to algorithms that rely on the full history. It has a linear computational complexity with respect to the window length $\tau$. Finally, depending on the search algorithm used to find the highest UCB index, the computational complexity can vary with respect to the number of arms $S$. For example, if we use the merge sort to sort the UCB indices of $S$ arms, BPRPC-SWUCB will have a complexity (with respect to the number of arms) of order $O(S \log S)$ \cite{Cormen09:ITA}.
\end{remark}
%-------------------------------------------------------------------> Section Numerical Analysis
\section{Numerical Analysis}
\label{sec:NumAna}
In this section, we investigate the empirical performance of BPRPC-SWUCB algorithm using the theoretical results obtained in this paper. To this end, we consider a computation offloading problem and draw the reward and cost of selecting each server based on the corresponding probability distributions derived in Section \ref{sec:ProFor}.

The setting of our simulation is as follows: (i) We consider a network consisting of three edge servers, i.e., $\lvert \mathcal{S} \rvert = 3$; (ii) As demonstrated in Section \ref{subsec:Reward}, at each round $\theta$, we sample the reward $r_{s,\theta}$ of selecting each server $s \in \mathcal{S}$ from a Bernoulli distribution with the piece-wise constant mean $\mu_{s,\theta}$; (iii) The distribution for the cost is derived in Section \ref{subsec:Cost}. We can rewrite the probability distribution (\ref{CostDist}) for the cost $c_{s,\theta}$ as 
%Single Column Version
\begin{equation} \label{CostDist2}
    \mathbb{P}(c_{s,\theta} = x) = \begin{cases}
    C_{x} (\frac{\rho_{s} - \lambda_{s,\theta}}{a_{s}}) e^{-(\frac{\rho_{s} - \lambda_{s,\theta}}{a_{s}})x}, \hspace{10mm} x \geq a^{\prime}_{s}+a^{\dprime}_{s} \\
    0, \hspace{46.5mm} x < a^{\prime}_{s}+a^{\dprime}_{s}
    \end{cases}
\end{equation}
%
%Double Column Version
% \begin{equation} \label{CostDist2}
%     \mathbb{P}(c_{s,\theta} = x) = \begin{cases}
%     C_{x} (\frac{\rho_{s} - \lambda_{s,\theta}}{a_{s}}) e^{-(\frac{\rho_{s} - \lambda_{s,\theta}}{a_{s}})x}, \hspace{4mm} x \geq a^{\prime}_{s}+a^{\dprime}_{s} \\
%     0, \hspace{38.5mm} x < a^{\prime}_{s}+a^{\dprime}_{s}
%     \end{cases}
% \end{equation}
%
where $C_{x}$ is a constant which depends on $x$. For a fixed $x$, $C_{x}$ is finite due to the summations being finite. The probability distribution in (\ref{CostDist2}) is similar to an exponential distribution with the support $[a^{\prime}_{s}+a^{\dprime}_{s},\infty]$. In our simulation, we consider an exponential distribution with $a_{s} = 1$ and $a^{\prime}_{s}, a^{\dprime}_{s} = 1/2$, $\forall s \in \mathcal{S}$, and with the piece-wise constant mean $\eta_{s,\theta}$; (iv) We consider at most $6$ change points in the mean reward or mean cost (including the one corresponding to the initial round). \textbf{Table \ref{Table:MABParams}} summarizes the change points in the expected value of the reward and cost variables for each server together with their values.
%-------------------------------------------------------------------> Table
%Single Column Version
\renewcommand{\arraystretch}{1.15}
\renewcommand{\tabcolsep}{1.4mm}
\begin{table}[b]
    {\footnotesize
    \begin{center}
    \caption{The list of mean rewards and mean costs associated with each server for different change points. The blank spaces represent there are no change points in those rounds, i.e., the mean remains the same as the previous change point.}
    \label{Table:MABParams}
    \begin{tabular}{|c|c|c|c|c|c|c|}
        \cline{2-7}
        \multicolumn{1}{c|}{} &
        \multicolumn{6}{c|}{Simulation Setting} \\
        \hline
        % \cline{2-7}
        \multicolumn{1}{|c|}{Change Point} &
        \multicolumn{2}{c|}{Server 1} &
        \multicolumn{2}{c|}{Server 2} &
        \multicolumn{2}{c|}{Server 3} \\
        % \hline
        %  &  &  &  &  &  & \\ 
\hline
    \begin{tabular}{r@{\;{=}\;}l}
        $\theta$ & $1$ \\
        $\theta$ & $500$ \\
        $\theta$ & $1000$ \\
        $\theta$ & $2000$ \\
        $\theta$ & $4000$ \\
        $\theta$ & $8000$
    \end{tabular} 
&
    \begin{tabular}{r@{\;{=}\;}l}
        $\mu_{1,1}$ & $0.5$ \\
        $\mu_{1,500}$ & $0.1$ \\
        $\mu_{1,1000}$ & $0.2$ \\
        $\mu_{1,2000}$ & $0.8$ \\
        $\mu_{1,4000}$ & $0.2$ \\
        \multicolumn{1}{l}{} 
    \end{tabular}
&
    \begin{tabular}{r@{\;{=}\;}l}
        $\eta_{1,1}$ & $1.1$ \\
        $\eta_{1,500}$ & $1.8$ \\
        \multicolumn{1}{l}{}  \\
        $\eta_{1,2000}$ & $1.2$ \\
        $\eta_{1,4000}$ & $1.5$ \\
        \multicolumn{1}{l}{}
  \end{tabular}
&
    \begin{tabular}{r@{\;{=}\;}l}
        $\mu_{2,1}$ & $0.4$ \\
        \multicolumn{1}{l}{}  \\
        $\mu_{2,1000}$ & $0.9$ \\
        $\mu_{2,2000}$ & $0.1$ \\
        $\mu_{2,4000}$ & $0.2$ \\
        $\mu_{2,8000}$ & $0.8$
    \end{tabular}
&
    \begin{tabular}{r@{\;{=}\;}l}
        $\eta_{2,1}$ & $1.2$ \\
        $\eta_{2,500}$ & $1.9$ \\
        $\eta_{2,1000}$ & $1.1$ \\
        $\eta_{2,2000}$ & $1.2$ \\
        $\eta_{2,4000}$ & $1.9$ \\
        $\eta_{2,8000}$ & $1.1$
  \end{tabular}
&
    \begin{tabular}{r@{\;{=}\;}l}
        $\mu_{3,1}$ & $0.3$ \\
        $\mu_{3,500}$ & $0.8$ \\
        $\mu_{3,1000}$ & $0.3$ \\
        \multicolumn{1}{l}{} \\
        $\mu_{3,4000}$ & $0.9$ \\
        $\mu_{3,8000}$ & $0.1$
    \end{tabular}
&
    \begin{tabular}{r@{\;{=}\;}l}
        $\eta_{3,1}$ & $1.4$ \\
        $\eta_{3,500}$ & $1.1$ \\
        $\eta_{3,1000}$ & $1.9$ \\
        \multicolumn{1}{l}{} \\
        $\eta_{3,4000}$ & $1.1$ \\
        $\eta_{3,8000}$ & $1.6$
  \end{tabular}\\
\hline
    \end{tabular}
    \end{center}
    }
\end{table}
%-------------------------------------------------------------------

% As explained in our analysis in Section \ref{sec:MAB}, we assess the performance of BPRPC-SWUCB by investigating the growth in its regret as a function of a growth in the budget $B$, or equivalently, $T(B)$.
% Moreover, 
We compare our algorithm with the following MAB-based policies:
\begin{itemize}
    \item  \textbf{KUBE:} We consider a variant of the algorithm KUBE that calculates the index for each $s \in \mathcal{S}$ as  $(\bar{r}_{\theta}(s) + \sqrt{(2 \log{\theta} ) / N_{\theta}(s)}) / \bar{c}_{\theta}(s)$, where $\bar{r}_{\theta}(s) = (1 / N_{\theta}(s)) \sum_{k=1}^{\theta} r_{s,k} \mathbbm{1}_{\{I_k = s\}}$, $\bar{c}_{\theta}(s) = (1 / N_{\theta}(s)) \sum_{k=1}^{\theta} c_{s,k} \mathbbm{1}_{\{I_k = s\}}$, and $N_{\theta}(s) = \sum_{k=1}^{\theta} \mathbbm{1}_{\{I_k = s\}}$ \cite{Tran12:KBO}.
    \item \textbf{UCB1:} It calculates the index for each $s \in \mathcal{S}$ as $((\sum_{k=1}^{\theta} (r_{I_{k},k} / c_{I_{k},k}) \mathbbm{1}_{\{I_k = s\}})/N_{\theta}(s)) + r_{\max} \sqrt{(\xi^{\prime} \log{\theta}) / N_{\theta}(s)}$, where $\xi^{\prime}$ is a tunable parameter \cite{Auer02:FTA}.

    \item \textbf{UCB-based algorithm:} We define a policy which explores similar to UCB1 but exploits similar to BPRPC-SWUCB. By implementing this algorithm, we can compare the performance of our algorithm with a general UCB-based algorithm. It calculates an index as $(\bar{r}_{\theta}(s) / \bar{c}_{\theta}(s)) + (r_{\max} / c_{\min}) \sqrt{(\xi^{\dprime} \log{\theta} ) / N_{\theta}(s)}$, where $\xi^{\dprime}$ is a tunable parameter.
    
    \item \textbf{UCB-BV1:} For each $s \in \mathcal{S}$, this algorithm calculates a UCB index as $(\bar{r}_{\theta}(s) / \bar{c}_{\theta}(s)) + ((1 + \frac{1}{c_{\min}})\sqrt{\frac{\log(\theta - 1)}{N_{\theta}(s)}}) / (c_{\min} - \sqrt{\frac{\log(\theta - 1)}{N_{\theta}(s)}})$ \cite{Ding13:MAB}.

    \item \textbf{$\varepsilon$-Greedy:}  At each round $\theta$, $\varepsilon$-Greedy chooses an arm uniformly at random with probability $\varepsilon$ and the best arm so far with probability $1-\varepsilon$ \cite{Auer02:FTA}. 
    % $\varepsilon = 1/\theta$
\end{itemize}
To be comparable with other algorithms, we chose the system variables so that to fulfill the prerequisites of the other algorithms. The tuned parameters used in our simulation are listed in \textbf{Table \ref{Table:PolicyParams}}. Note that, based on our problem setting, we have $r_{\max}, c_{\min} = 1$.
%-------------------------------------------------------------------> Table
%Single Column Version
\renewcommand{\arraystretch}{1.4}
\begin{table}[b]
    {\footnotesize
    \begin{center}
    \caption{The parameters of the different policies used in the simulation.}
    \label{Table:PolicyParams}
    \begin{tabular}{|c|c|c|c|c|c|c|}
    \cline{2-5}
        %\hline
        \multicolumn{1}{c|}{} &
        \multicolumn{4}{|c|}{Policy Setting} \\
        \hline
        Policy & UCB1 & BPRPC-SWUCB & $\varepsilon$-Greedy & UCB-based \\
        \hline
        \multirow{2}{*}{Parameters} & $\xi^{\prime} = 0.6$ & $\xi = 0.6$ & $\varepsilon = \frac{1}{\theta}$ & $\xi^{\dprime} = 0.6$ \\
        &  & $\tau = 2000$ & & \\
        \hline
    \end{tabular}
    \end{center}
    }
\end{table}
%-------------------------------------------------------------------

\textbf{Fig. \ref{Fig:All_Mean_Utilities}} depicts the evolution of the mean reward per mean cost for the three servers. The environment is dynamic in the sense that the optimal server in terms of the highest mean reward per mean cost changes over time. The change points can arise due to a change in mean reward, mean cost, or both. Note that, as mentioned before, the change points do not have to be identical; for example, at round $\theta = 1000$, the mean reward for server $1$ is changing while its mean cost remains fixed (Table \ref{Table:MABParams}).
%-------------------------------------------------------------------> Figure
%Single Column Version
% \begin{figure}[ht]
% \begin{center}
%     \begin{subfigure}[ht]{0.6\textwidth}
%     \centering
%         \includegraphics[width = 0.8\textwidth]{Images/Arm_Mean_Reward_Cost_BerExp_31_100dpi.png}
%         \caption{}
%         \label{SubFig:Mean_Reward_Cost}
%     \end{subfigure}
%     \\\vspace{4mm}%
%     \begin{subfigure}[ht]{0.6\textwidth}
%     \centering
%         \includegraphics[width = 0.8
%         \textwidth]{Images/Arm_Mean_Utilities_BerExp_30_100dpi.png}
%         \caption{}
%         \label{SubFig:All_Mean_Utilities}
%     \end{subfigure}
% \end{center}
% \caption{Characteristics of the dynamic environment; \ref{SubFig:Mean_Reward_Cost}: Evolution of the expected value of the reward and cost variables. \ref{SubFig:All_Mean_Utilities}: Changes in mean reward per mean cost for each server.}
% \label{Fig:SimulationSettings}
% \end{figure}
% %Single Column Version - one of them removed
\begin{figure}[t]
\begin{center}
\includegraphics[width=0.7\textwidth]{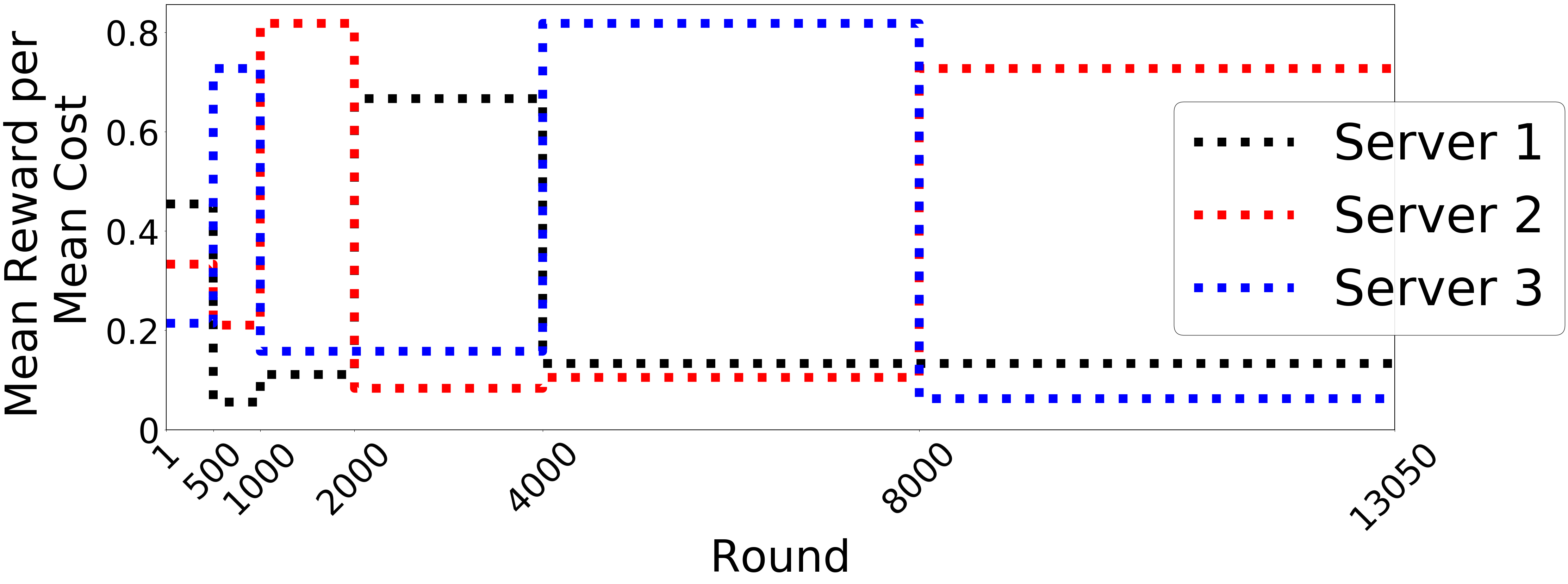}
\end{center}
\caption{Evolution of the mean reward per mean cost for each server.}
\label{Fig:All_Mean_Utilities}
\end{figure}
%-------------------------------------------------------------------v

\textbf{Fig. \ref{Fig:All_Regrets}} depicts the simulation results of running different policies to solve the computation offloading problem in the aforementioned network with a given budget $B = 15000$. It shows the trend of regret for each policy. To be comparable, we truncated the graph of all policies at the smallest stopping round among the different policies. As we see, BPRPC-SWUCB surpasses all other policies and is able to conform faster to abrupt changes in the environment. As a result, BPRPC-SWUCB has a smoother curve where does not exist sudden jumps in the regret, unlike other policies. The regret of other policies grows faster than BPRPC-SWUCB especially close to change points. Note that algorithms other than BPRPC-SWUCB fail in their performance due to their nature; they are designed to perform well in a stationary environment. 
%-------------------------------------------------------------------> Figure
%Single Column Version
\begin{figure}[b]
\begin{center}
\includegraphics[width=0.7\textwidth]{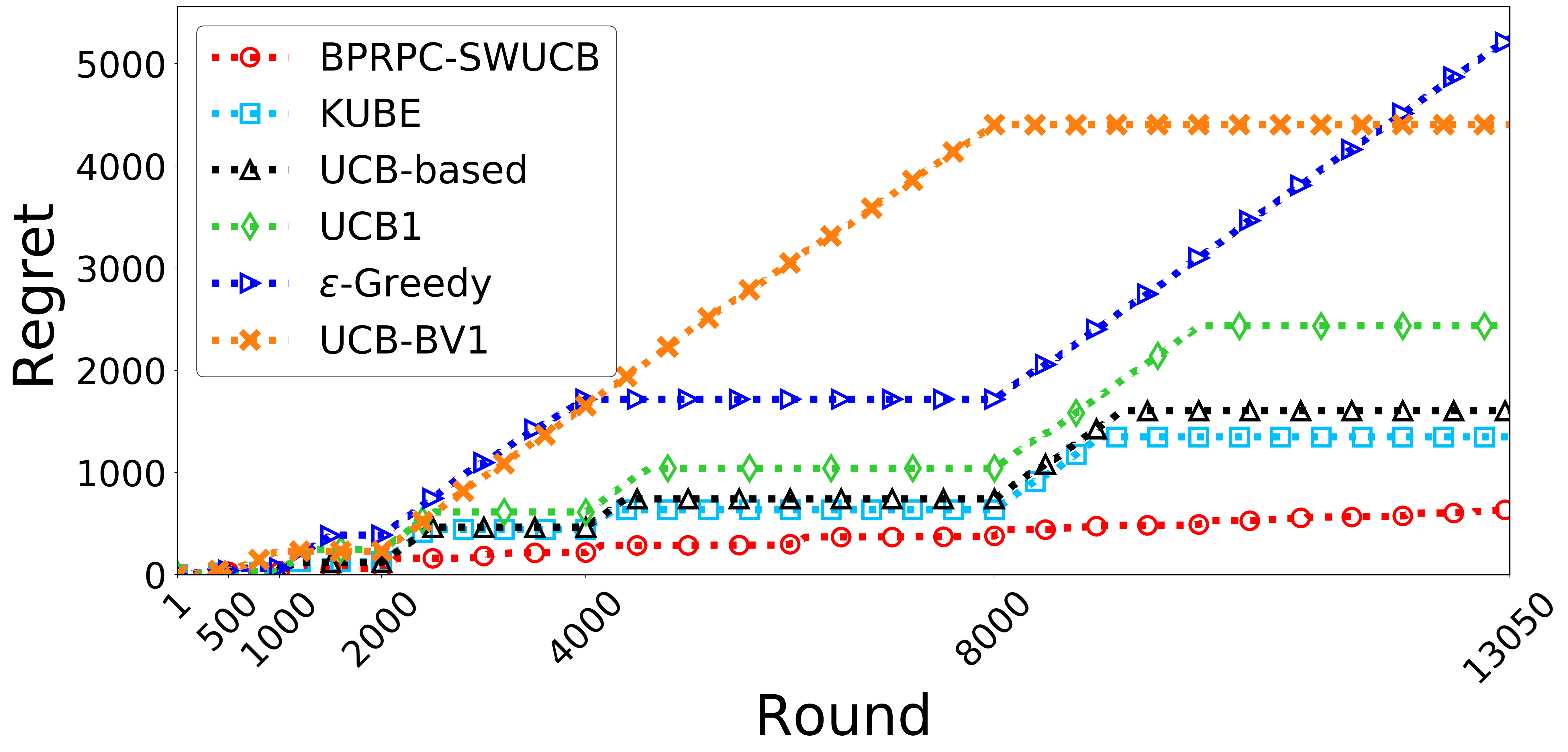}
\end{center}
\caption{Regret of different policies for a given same budget.}
% The computational offloading problem in a dynamic environment; 
\label{Fig:All_Regrets}
\end{figure}
%-------------------------------------------------------

\textbf{Fig. \ref{Fig:EmpMeanUtility}} depicts the highest mean reward per mean cost at each round, which is known to Oracle, and the empirically computed average reward per average cost of the chosen server by the other policies at each round. This figure illustrates well why BPRPC-SWUCB is performing better than other policies; it chooses the optimal server in more number of rounds (compared to other policies) due to its ability to detect the changes in the environment.
%-------------------------------------------------------------------> Figure
\begin{figure}[t]
\begin{center}
\includegraphics[width=0.6\textwidth]{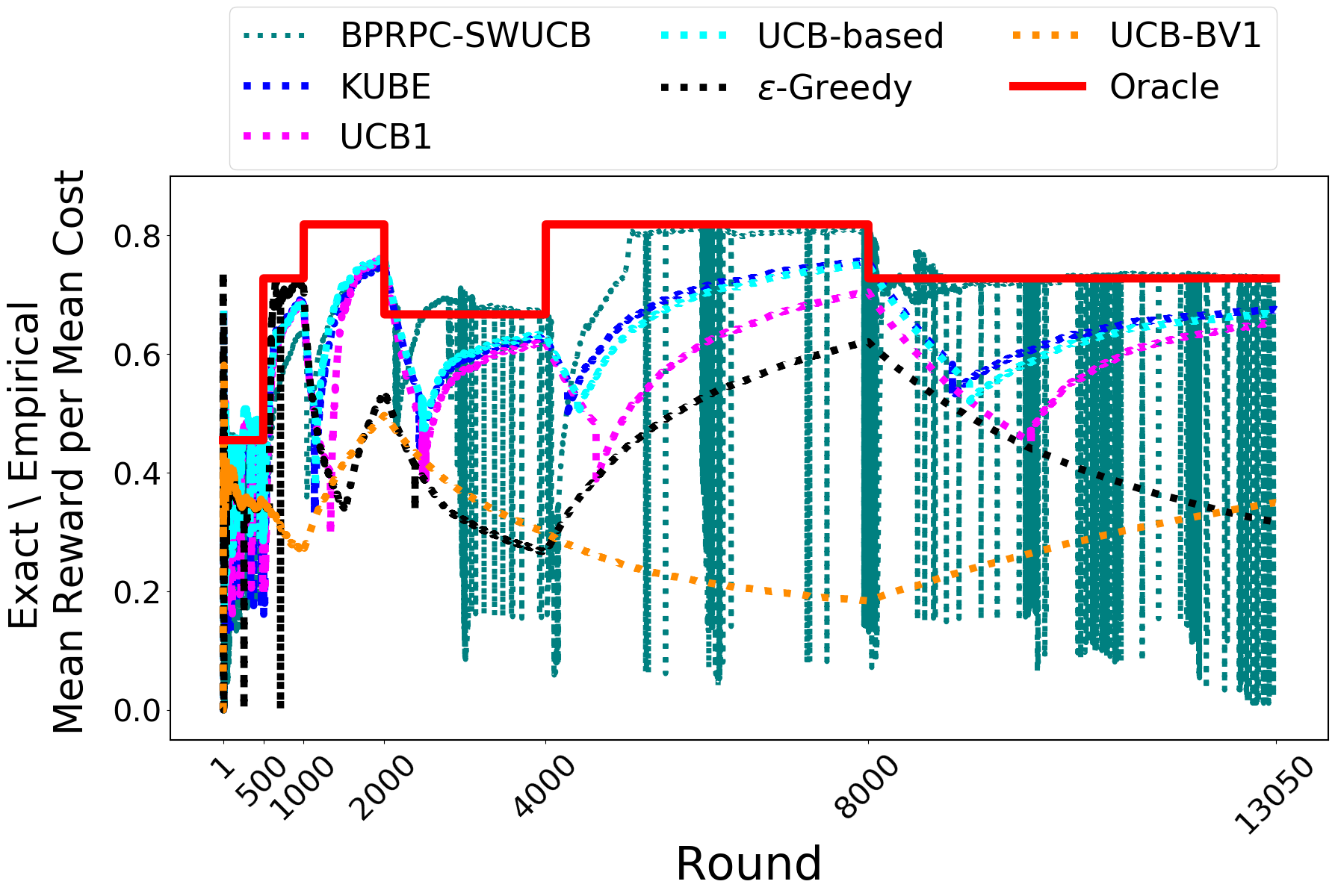}
\end{center}
\caption{The highest mean reward per mean cost at each round chosen by Oracle and the empirically computed average reward per average cost of the chosen server by different policies at each round.}
\label{Fig:EmpMeanUtility}
\end{figure}
%-------------------------------------------------------

\textbf{Fig. \ref{Fig:ServerChoice}} compares the performance of BPRPC-SWUCB with Oracle in terms of the choice of servers. As expected, in the first few rounds, mainly before the second change point at $\theta = 500$, BPRPC-SWUCB is investing more on exploring the servers to approximate the mean reward per mean cost of each server, and after $\theta = 500$, it detects the best server in most of the rounds even if there are sudden changes afterwards. This is due to using a sliding window $\tau$ which helps to detect the best server faster. We see that BPRPC-SWUCB has reasonably good performance compared to Oracle.
%-------------------------------------------------------------------> Figure
\begin{figure}[b]
\begin{center}
\includegraphics[width=0.85\textwidth]{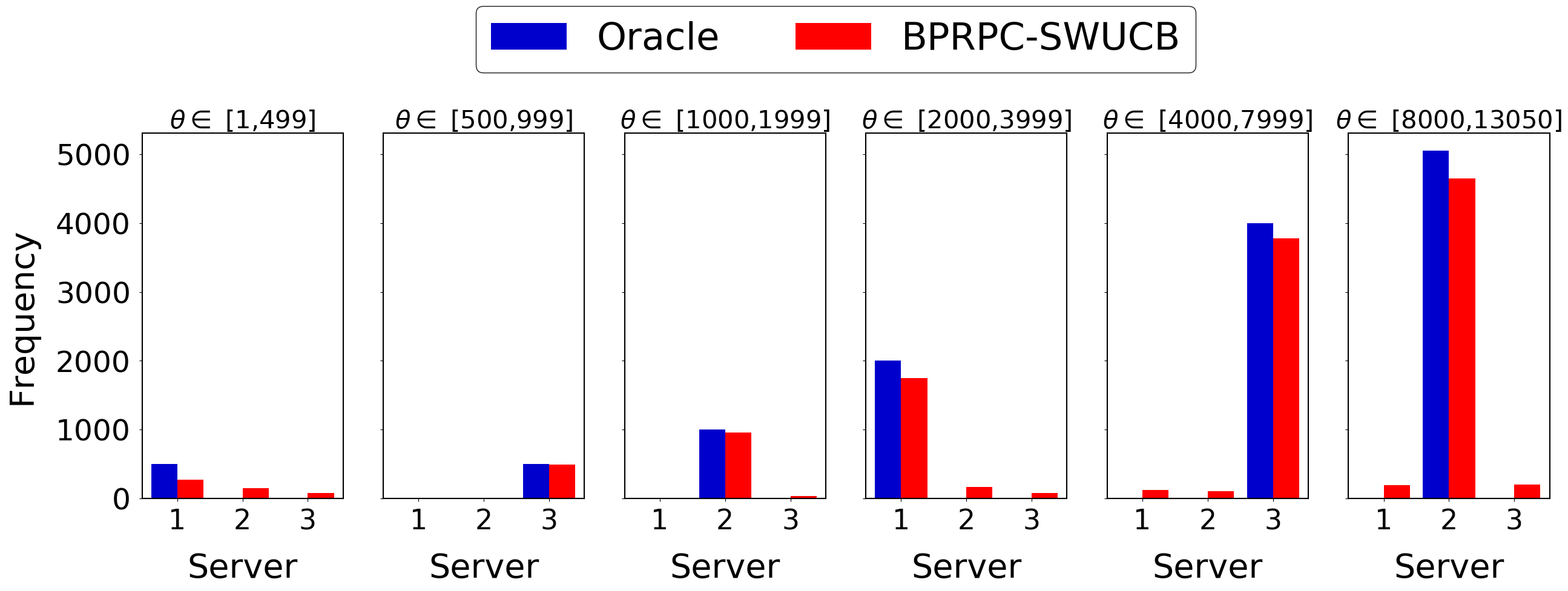}
\end{center}
\caption{Server choice for Oracle vs. BPRPC-SWUCB.}
\label{Fig:ServerChoice}
\end{figure}
%-------------------------------------------------------------------

%-------------------------------------------------------------------> Figure
\begin{figure}[ht]
\begin{center}
    \begin{subfigure}[ht]{0.48\textwidth}
    \centering
        \includegraphics[width = 0.9\textwidth]{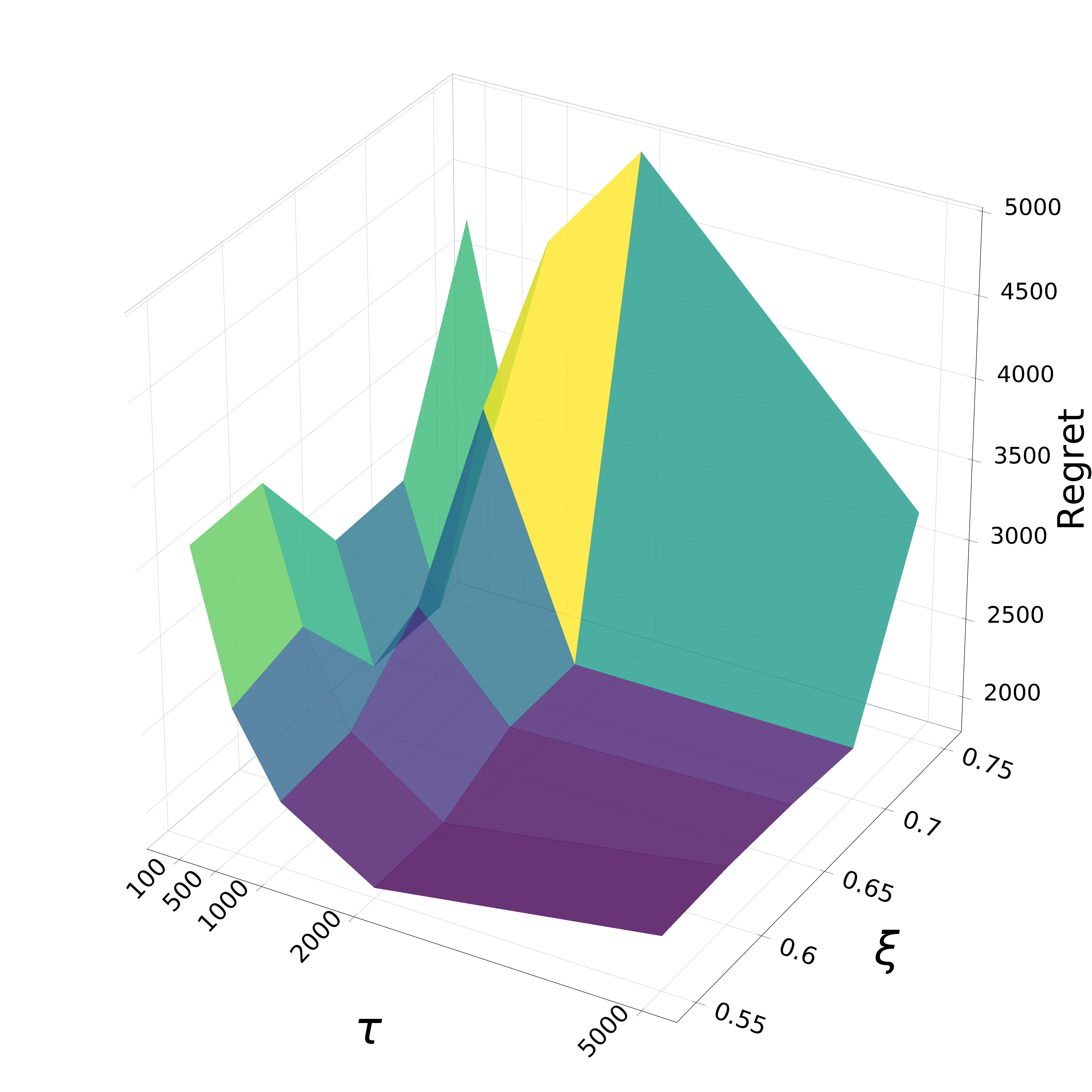}
        \caption{}
        \label{SubFig:Effect_of_Params1}
    \end{subfigure} %\hfill
    \begin{subfigure}[ht]{0.48\textwidth}
    \centering
    \vspace{18mm}
        \includegraphics[width = 1\textwidth]{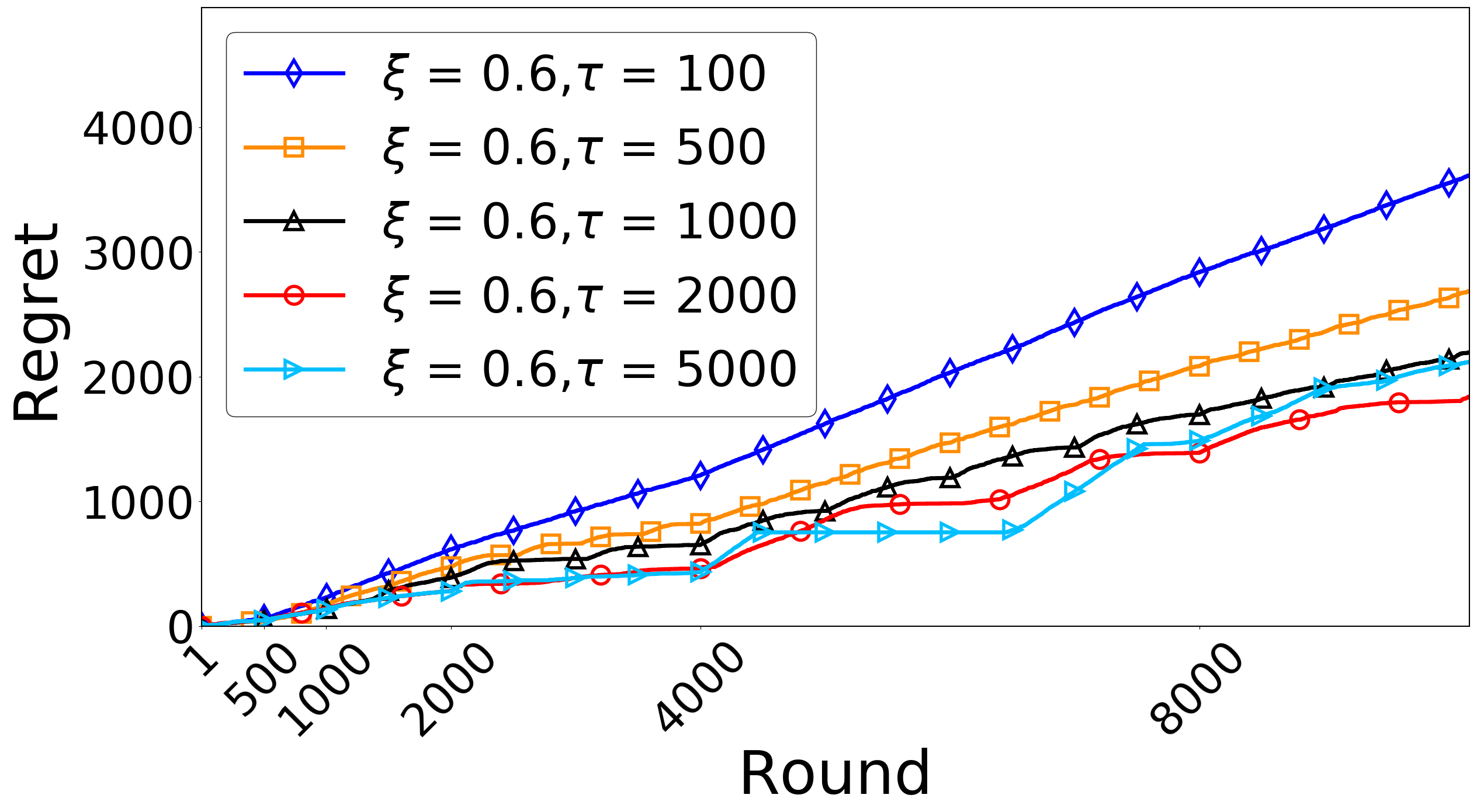}
        \vspace{5mm}
        \caption{}
        \label{SubFig:Effect_of_Params2}
    \end{subfigure}
\end{center}
\caption{The effect of parameters on the performance of BPRPC-SWUCB; \ref{SubFig:Effect_of_Params1}: Regret obtained for different $\xi$ and $\tau$. \ref{SubFig:Effect_of_Params2}: Regret for $\xi = 0.6$ and different window lengths $\tau$.}
% choices of 
\label{Fig:Effect_of_Params}
\end{figure}
As mentioned earlier, the performance of BPRPC-SWUCB highly depends on the choice of parameters. To better demonstrate this, we have shown the effect of parameters in \textbf{Fig. \ref{Fig:Effect_of_Params}}. Fig. \ref{SubFig:Effect_of_Params1} depicts an overview of the amount of regret obtained for different choices of the parameters, namely $\xi$ and the window length $\tau$. We see that for smaller values of $\xi$ and larger values of $\tau$ we have smaller regret. This graph is also obtained for a given budget $B = 15000$. Fig. \ref{SubFig:Effect_of_Params2} shows the trend of regret for a slice of the previous figure corresponding to $\xi = 0.6$. It clearly shows that for $\xi = 0.6$, a bigger $\tau$ results in a smaller regret. 
% However, when taking the storage efficiency into account, a smaller $\tau$ would be still beneficial.

%-------------------------------------------------------------------> Remark 4
\begin{remark} \label{rmk:parselection} Parameter Selection \\
Fig. \ref{Fig:Effect_of_Params} might appear different for a problem with different settings, for example, a problem with different change points, number of change points, number of arms, and so on. Hence, the parameters $\tau$ and $\xi$ should be chosen based on the given problem. Generally, $\xi$ controls the exploration power of the algorithm. A larger $\xi$ results in giving more importance to the exploration rather than exploiting the arm which shows promising results. In problems with more number of arms, a larger $\xi$ can be useful. The window length $\tau$ is chosen based on the number and frequency of change points. In general, selecting a smaller $\tau$ would be more suitable if change points occur often. Moreover, a smaller $\tau$ results in storage efficiency. In an environment where the system variables change seldom, we may choose a larger $\tau$.
\end{remark}
%-------------------------------------------------------------------
%-------------------------------------------------------------------> Section Conclusion
\section{Conclusion}
\label{sec:Con}
In this paper, we mainly focused on the computation offloading problem in a dynamic network under uncertainty; nonetheless, the theoretical results are applicable in a number of contexts. We derived the probability distribution of the required time for data transmission from the user's device to an edge server. Moreover, we analyzed the probability distribution of the required time for data processing in a server. By leveraging the aforementioned distributions, we derived the probability distribution of total required time and energy for the whole offloading process. In addition, we defined the reward and cost in terms of the required time and energy in each offloading round, respectively, and we derived the corresponding probability distributions. We then cast the server selection problem in the MAB framework. We developed a novel UCB-based algorithm, namely BPRPC-SWUCB, to solve the formulated problem. We analyzed BPRPC-SWUCB theoretically by proving an upper bound on its expected regret. The numerical results demonstrated that BPRPC-SWUCB performs well in a non-stationary environment.
%-------------------------------------------------------------------> Appendix
\section{Appendix}
\label{sec:App}
%
%-------------------------------------------------------------------> Proof Proposition 1
\subsection{Proof of Proposition \ref{Prop1}} 
\label{ProOneProof}
Fix a sink node $s$ and a round $\theta$. We will derive the probability distribution of $g_{s,\theta}$ by finding the joint distribution of the transmission time $g_{s,\theta}$ and the number of hops $H_{s}$. From the basics of probability theory we have
%Single Column Version
\begin{align} 
\label{JointDist}
\mathbb{P}(g_{s,\theta} = k) = \sum_{h = 1}^{h_{s,\max}} \mathbb{P}(g_{s,\theta} = k,H_{s} = h) = \sum_{h=1}^{h_{s,\max}} \mathbb{P}(g_{s,\theta} = k | H_{s} = h) \mathbb{P}(H_{s} = h).
\end{align}
%
%Double Column Version
% \begin{align} 
% \label{JointDist} \nonumber
% \mathbb{P}(g_{s,\theta} = k) &= \sum_{h = 1}^{h_{s,\max}} \mathbb{P}(g_{s,\theta} = k,H_{s} = h) \\
% &= \sum_{h=1}^{h_{s,\max}} \mathbb{P}(g_{s,\theta} = k | H_{s} = h) \mathbb{P}(H_{s} = h).
% \end{align}
%
The second term $\mathbb{P}(H_{s} = h)$ is given in (\ref{ProbHop}). The first term is derived in the following. For $k < h$, it is trivial that $\mathbb{P}(g_{s,\theta} = k | H_{s} = h) = 0$, as the number of attempts to transmit the data to a server cannot be less than the number of required hops. For $k \geq h$, it is a negative binomial distribution, as proved in the following. 
%Single Column Version
\begin{align} 
\label{ConditionalDist}
\mathbb{P}(g_{s,\theta} = k | H_{s} = h) \stackrel{\text{(a)}}{=} \mathbb{P}(K_{1} + K_{2} + \dots + K_{h} = k) \stackrel{\text{(b)}}{=} \binom{k-1}{h-1} p_{s,\theta}^{h} (1-p_{s,\theta})^{k-h},
\end{align}
%
%Double Column Version
% \begin{align} 
% \label{ConditionalDist} \nonumber
% \mathbb{P}(g_{s,\theta} = k | H_{s} = h) &\stackrel{\text{(a)}}{=} \mathbb{P}(K_{1} + K_{2} + \dots + K_{h} = k) \\
% &\stackrel{\text{(b)}}{=} \binom{k-1}{h-1} p_{s,\theta}^{h} (1-p_{s,\theta})^{k-h},
% \end{align}
%
where $(a)$ follows from the definition of $g_{s,\theta}$ and using the given condition $H_{s} = h$. Moreover, $(b)$ follows from the fact that the sum of $h$ independent and identical geometric random variables $K_{i}$ with the same parameter $p_{s,\theta}$ results in a negative binomial distribution with parameters $h$ and $p_{s,\theta}$ \cite{Wroughton13:DBB}. Note that this form of negative binomial distribution corresponds to the probability that $k$ number of trials is needed until the $h$-th success occur. Therefore,
%Single Column Version
\begin{equation} \label{ConditionalDistmain}
  \mathbb{P}(g_{s,\theta} = k | H_{s} = h) = \begin{cases}
    \binom{k-1}{h-1} p_{s,\theta}^{h} (1-p_{s,\theta})^{k-h}, \hspace{10mm} k \geq h \\
    0, \hspace{46.5mm} k < h
    \end{cases}
\end{equation}
%Double Column Version
% \begin{equation} \label{ConditionalDistmain}
%   \mathbb{P}(g_{s,\theta} = k | H_{s} = h) = \begin{cases}
%     \binom{k-1}{h-1} p_{s,\theta}^{h} (1-p_{s,\theta})^{k-h}, \hspace{2mm} k \geq h \\
%     0, \hspace{34mm} k < h
%     \end{cases}
% \end{equation}
%
Summarizing the above results, we can write an equivalent form of (\ref{JointDist}) as follows
\begin{equation}
\label{JointDist1}
       \mathbb{P}(g_{s,\theta} = k) = \sum_{h=1}^{\min{\{k,h_{s,\max}\}}} \mathbb{P}(g_{s,\theta} = k | H_{s} = h) \mathbb{P}(H_{s} = h).
\end{equation}
Thus, the first part of the proposition, i.e., (\ref{ProbTime}), follows by substituting (\ref{ProbHop}) and (\ref{ConditionalDist}) in (\ref{JointDist1}).

Since all the variables $K_{i}$ are independent and have the same expected value, it holds
\begin{equation}
\label{eq:Exp1}
\mathbb{E}[g_{s,\theta}] = \mathbb{E}[K_{i}] \mathbb{E}[H_{s}].
\end{equation}
We have $\mathbb{E}[K_{i}] = \frac{1}{p_{s,\theta}}$, $\forall i$. Therefore, the second part of the proposition, i.e., (\ref{ExpTime}), follows by substituting (\ref{ExpHops}) in (\ref{eq:Exp1}). 
% For $K_{i}$, $i=1,2, \dots$, w
%-------------------------------------------------------------------
%-------------------------------------------------------------------> Proof Proposition 2
\subsection{Proof of Proposition \ref{Prop2}}
\label{ProTwoProof}
We have the distribution of the delay time $d_{s,\theta}$ as the convolution of the two probability distributions of processing time $f_{s,\theta}$ and the transmission time $g_{s,\theta}$. 
From the definition of the reward, we have $r_{s,\theta} \in \{0,1\}$. 
Moreover, for any server $s \in \mathcal{S}$ and any round $\theta$ we have
%Single Column Version - new
\begin{align}
    &\hspace{-3mm} P_{s} = P(r_{s,\theta} = 1) = P(d_{s,\theta} \leq \delta) 
    \stackrel{\text{(a)}}{=}  \sum_{k=1}^{\floor{\delta}} \mathbb{P}(f_{s,\theta} \leq \delta - k) \mathbb{P}(g_{s,\theta} = k), \\
    &\hspace{-3mm} P_{f} = P(r_{s,\theta} = 0) = P(d_{s,\theta} \geq \delta) = 1 - P(d_{s,\theta} \leq \delta) 
    \stackrel{\text{(b)}}{=} 1 - \sum_{k=1}^{\floor{\delta}} \mathbb{P}(f_{s,\theta} \leq \delta - k) \mathbb{P}(g_{s,\theta} = k),
\end{align}
where $(a)$ and $(b)$ follow from the following facts; $d_{s,\theta}$ is a random variable which is the sum of two independent random variables $f_{s,\theta}$ and $g_{s,\theta}$. Note that, $f_{s,\theta}$ is a continuous random variable whereas $g_{s,\theta}$ is a discrete random variable. Moreover, we have $\delta - k \geq 0$ for $k \leq \floor{\delta}$ and $\mathbb{P}(f_{s,\theta} \leq \delta - k) = 0$ for $k > \floor{\delta}$. 
%
%Double Column Version
% \begin{align}
%     P_{s} &= P(r_{s,\theta} = 1) = P(d_{s,\theta} \leq \delta) \\
%     % &\stackrel{\text{(a)}}{=} \sum_{k=1}^{\infty} \mathbb{P}(f_{s,\theta} \leq \delta - k) \mathbb{P}(g_{s,\theta} = k) \\ 
%     &\stackrel{\text{(a)}}{=} \sum_{k=1}^{\floor{\delta}} \mathbb{P}(f_{s,\theta} \leq \delta - k) \mathbb{P}(g_{s,\theta} = k), \\
%     P_{f} &= P(r_{s,\theta} = 0) = 1 - P(d_{s,\theta} \leq \delta) \\
%     % &\stackrel{\text{(c)}}{=} 1 - \sum_{k=1}^{\infty} \mathbb{P}(f_{s,\theta} \leq \delta - k) \mathbb{P}(g_{s,\theta} = k) \\
%     &\stackrel{\text{(b)}}{=} 1 - \sum_{k=1}^{\floor{\delta}} \mathbb{P}(f_{s,\theta} \leq \delta - k) \mathbb{P}(g_{s,\theta} = k),
% \end{align}
%
% where $(a)$ and $(b)$ follow from the following facts; $d_{s,\theta}$ is a random variable which is the sum of two independent random variables $f_{s,\theta}$ and $g_{s,\theta}$. Note that, $f_{s,\theta}$ is a continuous random variable whereas $g_{s,\theta}$ is a discrete random variable. Moreover, we have $\delta - k \geq 0$ for $k \leq \floor{\delta}$ and $\mathbb{P}(f_{s,\theta} \leq \delta - k) = 0$ for $k > \floor{\delta}$. 
%
We can calculate the distributions $P_{s}$ and $P_{f}$ using the distributions of $f_{s,\theta}$ and $g_{s,\theta}$. Finally, we have $P_{s} + P_{f} = 1$. Hence, $r_{s,\theta}$ is a Bernoulli random variable with expected value (success probability) $P_{s}$. Thus, the result follows from Assumption \ref{Asp1:Nonstationarity}.
%-------------------------------------------------------------------
%-------------------------------------------------------------------> Proof Proposition 3
\subsection{Proof of Proposition \ref{Prop3}}
\label{ProThreeProof}
To prove the distribution, we first start by deriving the Cumulative Distribution Function (CDF) of the random variable cost. This is not a trivial task since the random variable $c_{s,\theta}$ is the result of linear combination of a continuous random variable $f_{s,\theta}$ and a discrete random variable $g_{s,\theta}$. In the following, $F_{Z}$ and $f_{Z}$ denote the CDF and the PDF of the random variable $Z$, respectively. Fix a server $s$ and an offloading round $\theta$. We have
%Single Column Version
\begin{align} \nonumber
    F_{c}(&c_{s,\theta} = x) 
    = \mathbb{P}(c_{s,\theta} \leq x)
    %= \mathbbm{P}(a_{s}f_{s,\theta} + a^{\prime}_{s}g_{s,\theta} + a^{\dprime}_{s} \leq x) \\
    = \sum_{k=1}^{\infty} \mathbb{P}(a_{s}f_{s,\theta} + a^{\prime}_{s}g_{s,\theta} + a^{\dprime}_{s} \leq x | g_{s,\theta} = k)\mathbb{P}(g_{s,\theta} = k) \\
    &= \sum_{k=1}^{\infty} \mathbb{P}(f_{s,\theta} \leq \frac{x-a^{\dprime}_{s}-a^{\prime}_{s}k}{a_{s}})\mathbb{P}(g_{s,\theta} = k) = \sum_{k=1}^{\infty} F_{f}(\frac{x-a^{\dprime}_{s}-a^{\prime}_{s}k}{a_{s}})\mathbb{P}(g_{s,\theta} = k).
\end{align}
%
%Double Column Version
% \begin{align}
%     F_{c}(&c_{s,\theta} = x) = \mathbb{P}(c_{s,\theta} \leq x) = \mathbbm{P}(a_{s}f_{s,\theta} + a^{\prime}_{s}g_{s,\theta} + a^{\dprime}_{s} \leq x) \\
%     &= \sum_{k=1}^{\infty} \mathbb{P}(a_{s}f_{s,\theta} + a^{\prime}_{s}g_{s,\theta} + a^{\dprime}_{s} \leq x | g_{s,\theta} = k)\mathbb{P}(g_{s,\theta} = k) \\
%     &= \sum_{k=1}^{\infty} \mathbb{P}(f_{s,\theta} \leq \frac{x-a^{\dprime}_{s}-a^{\prime}_{s}k}{a_{s}})\mathbb{P}(g_{s,\theta} = k) \\
%     &= \sum_{k=1}^{\infty} F_{f}(\frac{x-a^{\dprime}_{s}-a^{\prime}_{s}k}{a_{s}})\mathbb{P}(g_{s,\theta} = k).
% \end{align}
%
Taking the derivative of the above equation yields
%Single Column Version
\begin{align} \nonumber
    f_{c}(&c_{s,\theta} = x) 
    = \frac{d}{dx} F_{c}(c_{s,\theta} = x) = \sum_{k=1}^{\infty} \frac{d}{dx} F_{f}(\frac{x-a^{\dprime}_{s}-a^{\prime}_{s}k}{a_{s}
    })\mathbb{P}(g_{s,\theta} = k) \\
    &= \sum_{k=1}^{\infty} \frac{1}{a_{s}} f_{f}(\frac{x-a^{\dprime}_{s}-a^{\prime}_{s}k}{a_{s}}) \mathbb{P}(g_{s,\theta} = k) \stackrel{\text{$(\ast)$}}{=} \frac{1}{a_{s}}\sum_{k=1}^{\floor{\frac{x-a^{\dprime}_{s}}{a^{\prime}_{s}}}} f_{f}(\frac{x-a^{\dprime}_{s}-a^{\prime}_{s}k}{a_{s}}) \mathbb{P}(g_{s,\theta} = k),
\end{align}
%
%Double Column Version
% \begin{align}
%     f_{c}(&c_{s,\theta} = x) = \frac{d}{dx} F_{c}(c_{s,\theta} = x) \\
%     &= \sum_{k=1}^{\infty} \frac{d}{dx} F_{f}(\frac{x-a^{\dprime}_{s}-a^{\prime}_{s}k}{a_{s}})\mathbb{P}(g_{s,\theta} = k) \\
%     % &= \sum_{k=1}^{\infty} \frac{1}{a_{s}} f_{f}(\frac{x-a^{\dprime}_{s}-a^{\prime}_{s}k}{a_{s}}) \mathbb{P}(g_{s,\theta} = k) \\
%     &\stackrel{\text{$(\ast)$}}{=} \frac{1}{a_{s}}\sum_{k=1}^{\floor{\frac{x-a^{\dprime}_{s}}{a^{\prime}_{s}}}} f_{f}(\frac{x-a^{\dprime}_{s}-a^{\prime}_{s}k}{a_{s}}) \mathbb{P}(g_{s,\theta} = k),
% \end{align}
%
where $(\ast)$ follows from the fact that $f_{f}(\frac{x-a^{\dprime}_{s}-a^{\prime}_{s}k}{a_{s}}) = 0$ for $k > \floor{\frac{x-a^{\dprime}_{s}}{a^{\prime}_{s}}}$. The result follows by substituting the PDF of $f_{s,\theta}$ and the PMF of $g_{s,\theta}$, according to (\ref{ProbProcessTime}) and (\ref{ProbTime}), respectively. The expected value (\ref{ExpCost}) can be calculated by taking expectation from (\ref{cost}) and using the linearity property of the expected value operator.
%-------------------------------------------------------------------
%-------------------------------------------------------------------> Proof Lemma 1
\subsection{Proof of Lemma \ref{Lemma1}}
\label{LemmaOneProof}
For any policy $\pi$ (including the optimal policy), let $T^{\pi}(B)$ and $I^{\pi}_{\theta}$ denote its stopping round and its arm choice at round $\theta$, respectively. Moreover, let $B_{\theta}$ denote the budget left at round $\theta$ after pulling the arm $I^{\pi}_{\theta}$. Hence, $B_{\theta} = B - \sum_{k = 1}^{\theta} c_{I^{\pi}_{k},k}$. Inspired by \cite{Xia15:TSB}, we prove an upper bound on the expected cumulative reward of any policy $\pi$. We have
%Single Column Version
\begin{align} \nonumber
\mathbbm{E} {\Big[} \sum_{\theta=1}^{T^{\pi}(B)} r_{I^{\pi}_{\theta}, \theta} {\Big]} 
% \leq \mathbbm{E} {\Big[} \sum_{\theta = 1}^{T^{\pi}(B)} r_{I^{\pi}_{\theta}, \theta} {\Big]} + r_{\max}
&\stackrel{\text{$(a)$}}{\leq} \sum_{i=1}^{S} \sum_{\theta = 1}^{\infty} \mathbbm{E} [ r_{i, \theta} | I^{\pi}_{\theta} = i, B_{\theta} \geq 0] \mathbbm{P}(I^{\pi}_{\theta} = i, B_{\theta} \geq 0) + r_{\max} \\ \nonumber
&\stackrel{\text{$(b)$}}{\leq} \sum_{i=1}^{S} \sum_{\theta = 1}^{\infty}  \frac{\mu_{i^{\ast}_{\theta},\theta}}{\eta_{i^{\ast}_{\theta},\theta}}  \mathbbm{E} [ c_{i,\theta} | I^{\pi}_{\theta} = i, B_{\theta} \geq 0] \mathbbm{P}(I^{\pi}_{\theta} = i, B_{\theta} \geq 0)
% \\
% &\leq \frac{r_{\max}}{c_{\min}} \sum_{i=1}^{S} \sum_{\theta = 1}^{T^{\ast}(B)} \mathbbm{E}[ c_{i, \theta} | i^{\ast}_{\theta} = i] + r_{\max} \\
+ r_{\max} \\
& \leq \frac{r_{\max}}{c_{\min}} \mathbbm{E} {\Big[} \sum_{\theta=1}^{T^{\pi}(B)} c_{I^{\pi}_{\theta}, \theta} {\Big]} + r_{\max} \stackrel{\text{$(c)$}}{\leq} \frac{(B+c_{\min})r_{\max}}{c_{\min}},
\end{align}
%Double Column Version
% \begin{align}
% &\mathbbm{E} {\Big[} \sum_{\theta=1}^{T^{\pi}(B)} r_{I^{\pi}_{\theta}, \theta} {\Big]} 
% \leq \mathbbm{E} {\Big[} \sum_{\theta = 1}^{T^{\pi}(B)} r_{I^{\pi}_{\theta}, \theta} {\Big]} + r_{\max} \\
% &\stackrel{\text{$(a)$}}{\leq} \sum_{i=1}^{S} \sum_{\theta = 1}^{\infty} \mathbbm{E} [ r_{i, \theta} | I^{\pi}_{\theta} = i, B_{\theta} \geq 0] \mathbbm{P}(I^{\pi}_{\theta} = i, B_{\theta} \geq 0) + r_{\max} \\
% % &\stackrel{\text{$(b)$}}{\leq} \sum_{i=1}^{S} \sum_{\theta = 1}^{T^{\ast}(B)} \mathbbm{E} [ r_{i, \theta} | i^{\ast}_{\theta} = i] \mathbbm{P}(i^{\ast}_{\theta} = i) + r_{\max} \\
% &\stackrel{\text{$(b)$}}{\leq} \sum_{i=1}^{S} \sum_{\theta = 1}^{\infty}  \frac{\mu_{i^{\ast}_{\theta},\theta}}{\eta_{i^{\ast}_{\theta},\theta}}  \mathbbm{E} [ c_{i,\theta} | I^{\pi}_{\theta} = i, B_{\theta} \geq 0] \mathbbm{P}(I^{\pi}_{\theta} = i, B_{\theta} \geq 0) \\
% % &\leq \frac{r_{\max}}{c_{\min}} \sum_{i=1}^{S} \sum_{\theta = 1}^{T^{\ast}(B)} \mathbbm{E}[ c_{i, \theta} | i^{\ast}_{\theta} = i] + r_{\max} \\
% &+ \hspace{-0.5mm} r_{\max} \hspace{-0.5mm} \leq \frac{r_{\max}}{c_{\min}} \mathbbm{E} {\Big[} \sum_{\theta=1}^{T^{\pi}(B)} c_{I^{\pi}_{\theta}, \theta} {\Big]} + r_{\max} \stackrel{\text{$(c)$}}{\leq} \frac{(B+c_{\min})r_{\max}}{c_{\min}},
% \end{align}
%
where $(a)$ holds because of the definition of $B_{\theta}$, $(b)$ follows from $\frac{\mu_{i,\theta}}{\eta_{i,\theta}} \leq \frac{\mu_{i^{\ast}_{\theta},\theta}}{\eta_{i^{\ast}_{\theta},\theta}}$, $\forall i \in \mathcal{S}$, and $(c)$ holds because the algorithm stops before its total cost runs out of the budget $B$.
%-------------------------------------------------------------------
%-------------------------------------------------------------------> Proof Theorem 1
\subsection{Proof of Theorem \ref{Thm1}}
\label{TheoremOneProof}
Let $\tilde{N}_{T(B)}(i)$ denote the number of rounds arm $i$ has been played when it was not the optimal arm. Inspired by \cite{Garivier08:SWUCB} and \cite{Ding13:MAB}, we start by upper bounding the expected number of times a suboptimal arm was chosen given the stopping round $T(B)$. In the following, $\mathbbm{P}(X)$ and $\mathbbm{E}(X)$ represent the probability and expectation of the random variable $X$ under the policy of BPRPC-SWUCB, respectively. We first prove that for $i \in \mathcal{S}$ it holds.
%Single Column Version
\begin{align} \label{result2}
    &\mathbbm{E}[\tilde{N}_{T(B)}(i)|T(B)] \leq C(\tau,i) T(B) \frac{\log{(\tau)}}{\tau} + \tau \Upsilon_{T(B)} + 2\log^{2}(\tau),
\end{align}
%
%Double Column Version
% \begin{align} \label{result2} \nonumber
%     \mathbbm{E}[\tilde{N}_{T(B)}(i)|T(B)] \leq C(\tau,i) &T(B) \frac{\log{(\tau)}}{\tau} + \\
%     &+ \tau \Upsilon_{T(B)} + 2\log^{2}(\tau),
% \end{align}
%
where
%Single Column Version
\begin{align} \label{Cofsuboptimal2}
    C(\tau,i) = {\Bigg(}\frac{2(1 + \frac{r_{\max}}{c_{\min}}) + \Delta(i)}{c_{\min} \Delta(i)}{\Bigg)}^{2} r_{\max}^{2}\xi \frac{\normalceil{\frac{T(B)}{\tau}}}{\frac{T(B)}{\tau}} + \frac{4}{\log{(\tau)}} \ceil{\frac{\log{(\tau)}}{\log{(1 + 4 \sqrt{1 - (2 \xi)^{-1}})}}}.
\end{align}
%
%Double Column Version
% \begin{align} \label{Cofsuboptimal2} \nonumber
%     C(\tau,i) = {\Bigg(}&\frac{2(1 + \frac{r_{\max}}{c_{\min}}) + \Delta(i)}{c_{\min} \Delta(i)}{\Bigg)}^{2} r_{\max}^{2}\xi \frac{\normalceil{\frac{T(B)}{\tau}}}{\frac{T(B)}{\tau}} + \\
%     &+ \frac{4}{\log{(\tau)}} \ceil{\frac{\log{(\tau)}}{\log{(1 + 4 \sqrt{1 - (2 \xi)^{-1}})}}}.
% \end{align}
%
Let $J(\tau) = \hspace{-1mm}{\Bigg(}\frac{2(1 + \frac{r_{\max}}{c_{\min}}) + \Delta(i)}{c_{\min} \Delta(i)}{\Bigg)}^{2} r_{\max}^{2}\xi\log{(\tau)}$. Moreover, define $\Gamma(\tau)$ as
%Single Column Version
\begin{align}
\Gamma(\tau)\hspace{-1mm}=\hspace{-1mm}{\Big\{}\theta \in \hspace{-0.5mm} \{S+1, \dots, T(B)\} {\Big|} \mu_{i,j}\hspace{-1mm}=\hspace{-1mm}\mu_{i,\theta} \hspace{1mm} \& \hspace{1mm} \eta_{i,j}\hspace{-1mm}=\hspace{-1mm}\eta_{i,\theta}, \forall i \in \hspace{-0.5mm} \{1, \dots, S\} \hspace{1mm} \& \hspace{1mm} \forall j \hspace{1mm} \text{s.t.} \hspace{1mm} \theta - \tau < j \leq \theta {\Big\}}.
\end{align}
%
%Double Column Version
% \begin{align}
% \Gamma(\tau) = {\Big\{}\theta \in &\{S+1, \dots, T(B)\} \hspace{0.5mm} {\Big|} \hspace{0.5mm} \mu_{i,j} = \mu_{i,\theta} \hspace{0.5mm} \& \hspace{0.5mm} \eta_{i,j} = \eta_{i,\theta}, \\
% &\forall i \in \{1, \dots, S\} \hspace{1mm} \& \hspace{1mm} \forall j \hspace{1mm} \text{s.t.} \hspace{1mm} \theta - \tau < j \leq \theta {\Big\}}.
% \end{align}
% 
We have the following \cite{Garivier08:SWUCB}
%Single Column Version
\begin{align} 
\nonumber
\tilde{N}_{T(B)}(i) &= 1 + \sum_{\theta=S+1}^{T(B)} \mathbbm{1}_{\{I_\theta = i \neq i^{\ast}_{\theta}\}} \leq 1 + \sum_{\theta=1}^{T(B)} \mathbbm{1}_{\{I_\theta = i \neq i^{\ast}_{\theta}, N_{\theta}(\tau,i) < J(\tau)\}} + \sum_{\theta=S+1}^{T(B)}  \mathbbm{1}_{\{I_\theta = i \neq i^{\ast}_{\theta}, N_{\theta}(\tau,i) \geq J(\tau)\}} \\ 
% \nonumber
%  &\leq 1 + \ceil{\frac{T(B)}{\tau}} J(\tau) + \sum_{\theta=S+1}^{T(B)} \mathbbm{1}_{\{I_\theta = i \neq i^{\ast}_{\theta}, N_{\theta}(\tau,i) \geq J(\tau)\}} \\ 
 &\stackrel{\text{$(\ast)$}}{\leq} 1 + \ceil{\frac{T(B)}{\tau}} J(\tau) + \tau \Upsilon_{T(B)} + \sum_{\theta \in \Gamma(\tau)}^{} \mathbbm{1}_{\{I_\theta = i \neq i^{\ast}_{\theta}, N_{\theta}(\tau,i) \geq J(\tau)\}} \label{avali},
\end{align}
%
%Double Column Version
% \begin{align} 
% \nonumber
% &\tilde{N}_{T(B)}(i) = 1 + \sum_{\theta=S+1}^{T(B)} \mathbbm{1}_{\{I_\theta = i \neq i^{\ast}_{\theta}\}} \leq 1 + \\ \nonumber
% &+ \hspace{-1mm} \sum_{\theta=1}^{T(B)} \hspace{-1mm} \mathbbm{1}_{\{I_\theta = i \neq i^{\ast}_{\theta}, N_{\theta}(\tau,i) < J(\tau)\}} + \hspace{-2mm} \sum_{\theta=S+1}^{T(B)} \hspace{-1mm} \mathbbm{1}_{\{I_\theta = i \neq i^{\ast}_{\theta}, N_{\theta}(\tau,i) \geq J(\tau)\}} \\ \nonumber
% %  &\stackrel{\text{$(\ast)$}}{\leq} 1 + \ceil{\frac{T(B)}{\tau}} J(\tau) + \sum_{\theta=S+1}^{T(B)} \mathbbm{1}_{\{I_\theta = i \neq i^{\ast}_{\theta}, N_{\theta}(\tau,i) \geq J(\tau)\}} \\ \nonumber
%  &\stackrel{\text{$(\ast)$}}{\leq} 1 + \ceil{\frac{T(B)}{\tau}} J(\tau) + \tau \Upsilon_{T(B)} +\\ 
%  &\hspace{30mm} + \sum_{\theta \in \Gamma(\tau)}^{} \mathbbm{1}_{\{I_\theta = i \neq i^{\ast}_{\theta}, N_{\theta}(\tau,i) \geq J(\tau)\}} \label{avali},
% \end{align}
%
where ($\ast$) follows from the Lemma (25) in \cite{Garivier08:SWUCB}. For $\theta \in \Gamma(\tau)$, we have
%Single Column Version
\begin{align} \nonumber
{\{I_\theta = i \neq i^{\ast}_{\theta}, N_{\theta}(\tau,i) \geq J(\tau)\}} &\subset \underbrace{\{\frac{\bar{r}_{\theta}(\tau,i)}{\bar{c}_{\theta}(\tau,i)} > \frac{\mu_{i,\theta}}{\eta_{i,\theta}} + E_{\theta}(\tau,i)\}}_{1} \cup \underbrace{\{\frac{\bar{r}_{\theta}(\tau,i^{\ast}_{\theta})}{\bar{c}_{\theta}(\tau,i^{\ast}_{\theta})} < \frac{\mu_{i^{\ast}_{\theta},\theta}}{\eta_{i^{\ast}_{\theta},\theta}} - E_{\theta}(\tau,i^{\ast}_{\theta})\}}_{2} \\
&\hspace{1mm} \cup \underbrace{ \{\frac{\mu_{i^{\ast}_{\theta},\theta}}{\eta_{i^{\ast}_{\theta},\theta}} - \frac{\mu_{i,\theta}}{\eta_{i,\theta}} < 2E_{\theta}(\tau,i), N_{\theta}(\tau,i) \geq J(\tau)\}}_{3}.
\end{align}
%
%Double Column Version
% \begin{align}
% &{\{I_\theta = i \neq i^{\ast}_{\theta}, N_{\theta}(\tau,i) \hspace{-0.5mm} \geq \hspace{-0.5mm} J(\tau)\}} \hspace{-0.5mm} \subset \hspace{-0.5mm} \underbrace{\{\frac{\bar{r}_{\theta}(\tau,i)}{\bar{c}_{\theta}(\tau,i)} \hspace{-0.5mm} > \hspace{-0.5mm} \frac{\mu_{i,\theta}}{\eta_{i,\theta}} + \hspace{-0.5mm} E_{\theta}(\tau,i)\}}_{1} \\
% &\hspace{10mm} \cup \underbrace{\{\frac{\bar{r}_{\theta}(\tau,i^{\ast}_{\theta})}{\bar{c}_{\theta}(\tau,i^{\ast}_{\theta})} < \frac{\mu_{i^{\ast}_{\theta},\theta}}{\eta_{i^{\ast}_{\theta},\theta}} - E_{\theta}(\tau,i^{\ast}_{\theta})\}}_{2} \\
% &\hspace{10mm} \cup \underbrace{ \{\frac{\mu_{i^{\ast}_{\theta},\theta}}{\eta_{i^{\ast}_{\theta},\theta}} - \frac{\mu_{i,\theta}}{\eta_{i,\theta}} < 2E_{\theta}(\tau,i), N_{\theta}(\tau,i) \geq J(\tau)\}}_{3}.
% \end{align}
%
For the Event $3$, we have
%Single Column Version
\begin{align}
E_{\theta}(\tau,i) =\frac{(1 + \frac{r_{\max}}{c_{\min}}) r_{\max} \sqrt{\frac{\xi \log{(min\{\theta,\tau\})}}{N_{\theta}(\tau,i)}}}{c_{\min} - r_{\max} \sqrt{\frac{\xi \log{(min\{\theta,\tau\})}}{N_{\theta}(\tau,i)}}} \leq \frac{(1 + \frac{r_{\max}}{c_{\min}}) r_{\max} \sqrt{\frac{\xi \log{(\tau)}}{J(\tau)}}}{c_{\min} - r_{\max} \sqrt{\frac{\xi \log{(\tau)}}{J(\tau)}}} = \frac{\Delta(i)}{2}.
\end{align}
%
%Double Column Version
% \begin{align}
% E_{\theta}(\tau,i) &=\frac{(1 + \frac{r_{\max}}{c_{\min}}) r_{\max} \sqrt{\frac{\xi \log{(min\{\theta,\tau\})}}{N_{\theta}(\tau,i)}}}{c_{\min} - r_{\max} \sqrt{\frac{\xi \log{(min\{\theta,\tau\})}}{N_{\theta}(\tau,i)}}} \\
% &\leq \frac{(1 + \frac{r_{\max}}{c_{\min}}) r_{\max} \sqrt{\frac{\xi \log{(\tau)}}{J(\tau)}}}{c_{\min} - r_{\max} \sqrt{\frac{\xi \log{(\tau)}}{J(\tau)}}} = \frac{\Delta(i)}{2}.
% \end{align}
%
Therefore, the Event $3$ never occurs. Upper bound for the Events $1$ and $2$ are similar and we show only for Event $1$. Note that if Event $1$ occurs, it implies that at least one of the two following inequalities happens.
\begin{align} \label{oneofthem1}
\bar{r}_{\theta}(\tau,i) > \mu_{i,\theta} + e_{\theta}(\tau,i),
\end{align}
or
\begin{align} \label{oneofthem2}
\bar{c}_{\theta}(\tau,i) < \eta_{i,\theta} - e_{\theta}(\tau,i), 
\end{align}
where 
\begin{equation} 
e_{\theta}(\tau,i) = r_{\max} \sqrt{\frac{\xi\log{(\min\{\theta,\tau\})}}{N_{\theta}(\tau,i)}}. 
\end{equation} 
To prove this, assume none of them happens. Therefore, we have \cite{Ding13:MAB}
%Single Column Version
\begin{align} \nonumber
\frac{\bar{r}_{\theta}(\tau,i)}{\bar{c}_{\theta}(\tau,i)} - \frac{\mu_{i,\theta}}{\eta_{i,\theta}} &= \frac{(\bar{r}_{\theta}(\tau,i) - \mu_{i,\theta}) \eta_{i,\theta} + (\eta_{i,\theta} - \bar{c}_{\theta}(\tau,i)) \mu_{i,\theta}}{\bar{c}_{\theta}(\tau,i) \eta_{i,\theta}}
\leq \frac{e_{\theta}(\tau,i)}{\bar{c}_{\theta}(\tau,i)} + \frac{e_{\theta}(\tau,i) \mu_{i,\theta}}{\bar{c}_{\theta}(\tau,i) \eta_{i,\theta}} \\ &\leq \frac{e_{\theta}(\tau,i)}{c_{\min} - e_{\theta}(\tau,i)} + \frac{e_{\theta}(\tau,i) r_{\max}}{(c_{\min} - e_{\theta}(\tau,i)) c_{\min}} = E_{\theta}(\tau,i).
\end{align}
%
%Double Column Version
% \begin{align} 
% &\frac{\bar{r}_{\theta}(\tau,i)}{\bar{c}_{\theta}(\tau,i)} - \frac{\mu_{i,\theta}}{\eta_{i,\theta}} = \frac{(\bar{r}_{\theta}(\tau,i) - \mu_{i,\theta}) \eta_{i,\theta} + (\eta_{i,\theta} - \bar{c}_{\theta}(\tau,i)) \mu_{i,\theta}}{\bar{c}_{\theta}(\tau,i) \eta_{i,\theta}} \\
% % &\leq \frac{e_{\theta}(\tau,i)}{\bar{c}_{\theta}(\tau,i)} + \frac{e_{\theta}(\tau,i) \mu_{i,\theta}}{\bar{c}_{\theta}(\tau,i) \eta_{i,\theta}} \\
% &\leq \frac{e_{\theta}(\tau,i)}{c_{\min} - e_{\theta}(\tau,i)} + \frac{e_{\theta}(\tau,i) r_{\max}}{(c_{\min} - e_{\theta}(\tau,i)) c_{\min}} = E_{\theta}(\tau,i).
% \end{align}
%
Hence, we upper bound the probability of (\ref{oneofthem1}) and (\ref{oneofthem2}). Using Corollary (21) in \cite{Garivier08:SWUCB} for any $\nu > 0$ we have
%Single Column Version
\begin{align} 
\label{dovomi}
\mathbbm{P}(\bar{r}_{\theta}(\tau,i) > \mu_{i,\theta} + e_{\theta}(\tau,i)) &\leq \ceil{\frac{\log{(\min\{\theta,\tau\})}}{\log{(1+\nu)}}} (\min\{\theta,\tau\})^{-2\xi(1-\frac{\nu^{2}}{16})},
\end{align}
%
%Double Column Version
% \begin{align*} 
% % \label{dovomi}
% \mathbbm{P}(\bar{r}_{\theta}(\tau,i) > \mu_{i,\theta} &+ e_{\theta}(\tau,i)) \leq \\ &\ceil{\frac{\log{(\min\{\theta,\tau\})}}{\log{(1+\nu)}}} (\min\{\theta,\tau\})^{-2\xi(1-\frac{\nu^{2}}{16})},
% \end{align*}
%
and
%Single Column Version
\begin{align} \label{sevomi}
\mathbbm{P}(\bar{c}_{\theta}(\tau,i) < \eta_{i,\theta} - e_{\theta}(\tau,i)) &\leq \ceil{\frac{\log{(\min\{\theta,\tau\})}}{\log{(1+\nu)}}} (\min\{\theta,\tau\})^{-2\xi(1-\frac{\nu^{2}}{16})}.
\end{align}
%
%Double Column Version
% \begin{align*} 
% % \label{sevomi}
% \mathbbm{P}(\bar{c}_{\theta}(\tau,i) < \eta_{i,\theta} &- e_{\theta}(\tau,i)) \leq \\
% &\ceil{\frac{\log{(\min\{\theta,\tau\})}}{\log{(1+\nu)}}} (\min\{\theta,\tau\})^{-2\xi(1-\frac{\nu^{2}}{16})}.
% \end{align*}
%
%Single Column Version
For the Event $2$, we have similar results as follows.
%Single Column Version
\begin{align} 
\label{charomi}
\mathbbm{P}(\bar{r}_{\theta}(\tau,i^{\ast}_{\theta}) > \mu_{i^{\ast}_{\theta},\theta} + e_{\theta}(\tau,i^{\ast}_{\theta})) &\leq \ceil{\frac{\log{(\min\{\theta,\tau\})}}{\log{(1+\nu)}}} (\min\{\theta,\tau\})^{-2\xi(1-\frac{\nu^{2}}{16})},
\end{align}
%
%Double Column Version
% \begin{align} 
% \label{charomi} \nonumber
% &\mathbbm{P}(\bar{r}_{\theta}(\tau,i^{\ast}_{\theta}) > \mu_{i^{\ast}_{\theta},\theta} + e_{\theta}(\tau,i^{\ast}_{\theta})) \leq \\
% &\ceil{\frac{\log{(\min\{\theta,\tau\})}}{\log{(1+\nu)}}} (\min\{\theta,\tau\})^{-2\xi(1-\frac{\nu^{2}}{16})},
% \end{align}
%
and
%Single Column Version
\begin{align} 
\label{panjomi}
\mathbbm{P}(\bar{c}_{\theta}(\tau,i^{\ast}_{\theta}) < \eta_{i^{\ast}_{\theta},\theta} - e_{\theta}(\tau,i^{\ast}_{\theta})) &\leq \ceil{\frac{\log{(\min\{\theta,\tau\})}}{\log{(1+\nu)}}} (\min\{\theta,\tau\})^{-2\xi(1-\frac{\nu^{2}}{16})}.
\end{align}
Choosing $\nu = 4 \sqrt{1 - \frac{1}{2 \xi}}$ as suggested in \cite{Garivier08:SWUCB}, combinig (\ref{avali}) and (\ref{dovomi})-(\ref{panjomi}), and taking expectation result in
%
%Double Column Version
% \begin{align} 
% \label{panjomi} \nonumber
% &\mathbbm{P}(\bar{c}_{\theta}(\tau,i^{\ast}_{\theta}) < \eta_{i^{\ast}_{\theta},\theta} - e_{\theta}(\tau,i^{\ast}_{\theta})) \leq \\
% &\ceil{\frac{\log{(\min\{\theta,\tau\})}}{\log{(1+\nu)}}} (\min\{\theta,\tau\})^{-2\xi(1-\frac{\nu^{2}}{16})}.
% \end{align}
%
%Double Column Version
% For the Event $2$, we have similar results.
%
% Choosing $\nu = 4 \sqrt{1 - \frac{1}{2 \xi}}$ as suggested in \cite{Garivier08:SWUCB} and combinig the above results in (\ref{avali}) yield
%Single Column Version
\begin{align}
\mathbbm{E}[\tilde{N}_{T(B)}(i) | T(B)] \leq 1 + \ceil{\frac{T(B)}{\tau}} J(\tau) + \tau \Upsilon_{T(B)} + 4 \sum_{\theta=1}^{T(B)}\frac{\ceil{\frac{\log{(\min\{\theta,\tau\})}}{\log{(1+\nu)}}}}{\min\{\theta,\tau\}}.
\end{align}
%
%Double Column Version
% \begin{align} \nonumber
% \mathbbm{E}[&\tilde{N}_{T(B)}(i) | T(B)] \leq \\
% &\hspace{5mm}1 + \ceil{\frac{T(B)}{\tau}} J(\tau) + \tau \Upsilon_{T(B)} + 4 \sum_{\theta=1}^{T(B)}\frac{\ceil{\frac{\log{(\min\{\theta,\tau\})}}{\log{(1+\nu)}}}}{\min\{\theta,\tau\}}.
% \end{align}
%
We achieve the equation (\ref{result2}) using the following \cite{Garivier08:SWUCB}
%Single Column Version
\begin{align}
\sum_{\theta = S+1}^{T(B)} \frac{\log{(\min\{\theta,\tau\})}}{\min\{\theta,\tau\}} \leq \sum_{\theta=2}^{\tau} \frac{\log{(\theta)}}{\theta} + \sum_{\theta=1}^{T(B)} \frac{\log{(\tau)}}{\tau} \leq \frac{1}{2} \log^{2}{(\tau)} + \frac{T(B) \log{(\tau)}}{\tau}.
\end{align}
%
%Double Column Version
% \begin{align}
% \sum_{\theta = S+1}^{T(B)} \frac{\log{(\min\{\theta,\tau\})}}{\min\{\theta,\tau\}} &\leq \sum_{\theta=2}^{\tau} \frac{\log{(\theta)}}{\theta} + \sum_{\theta=1}^{T(B)} \frac{\log{(\tau)}}{\tau} \\
% &\leq \frac{1}{2} \log^{2}{(\tau)} + \frac{T(B) \log{(\tau)}}{\tau}.
% \end{align}
%

We rewrite the expected regret as
%Single Column Version
\begin{align} \label{ProblemSolved}
    \mathbbm{E}[R_{T(B)}] 
    &= \underbrace{{\Bigg (} \mathbbm{E}{\Bigg [} \sum_{\theta=1}^{T^{\ast}(B)} r_{i^{\ast}_{\theta},\theta} {\Bigg ]} - \mathbbm{E}{\Bigg [} \sum_{\theta=1}^{T(B)} r_{i^{\ast}_{\theta},\theta} {\Bigg ]} {\Bigg )}}_{1} + \underbrace{{\Bigg (} \mathbbm{E}{\Bigg [} \sum_{\theta=1}^{T(B)} r_{i^{\ast}_{\theta},\theta} {\Bigg ]} - \mathbbm{E}{\Bigg [} \sum_{\theta=1}^{T(B)} r_{I_{\theta},\theta} {\Bigg ]} {\Bigg )}}_{2}.
\end{align}

%
%Double Column Version
%
We bound each part separately. For the first term in Part 1, the approach is similar to the proof of Lemma \ref{Lemma1}. However, here we bound the difference between the total reward obtained by playing the optimal arm permanently but with two different stopping rounds: (i) The stopping round corresponding to the optimal policy and (ii) the stopping round of our policy. As before, we define $B_{\theta} = B - \sum_{k = 1}^{\theta} c_{I^{\pi}_{k},k}$. We have
%Single Column Version
\begin{align} \label{eq:inregret_first} \nonumber
\mathbbm{E} {\Bigg [} &\sum_{\theta=1}^{T^{\ast}(B)} r_{i^{\ast}_{\theta},\theta} {\Bigg ]} 
- \mathbbm{E} {\Bigg [} \sum_{\theta=1}^{T(B)} r_{i^{\ast}_{\theta},\theta} {\Bigg ]}
\stackrel{\text{$(a)$}}{\leq} 
\sum_{\theta = 1}^{\infty} \mathbbm{E} [ r_{i^{\ast}_{\theta}, \theta} | i^{\ast}_{\theta} = i^{\ast}_{\theta}, B_{\theta} \geq 0] \mathbbm{P}(i^{\ast}_{\theta} = i^{\ast}_{\theta}, B_{\theta} \geq 0) + r_{\max} \\ \nonumber
&- \sum_{\theta = 1}^{\infty} \mathbbm{E} [ r_{i^{\ast}_{\theta}, \theta} | I_{\theta} = i^{\ast}_{\theta}, B_{\theta} \geq c_{\max}] \mathbbm{P}(I_{\theta} = i^{\ast}_{\theta}, B_{\theta} \geq c_{\max}) \\ \nonumber
&=
\sum_{\theta = 1}^{\infty}  \frac{\mu_{i^{\ast}_{\theta},\theta}}{\eta_{i^{\ast}_{\theta},\theta}}  \mathbbm{E} [ c_{i^{\ast}_{\theta},\theta} | i^{\ast}_{\theta} = i^{\ast}_{\theta}, B_{\theta} \geq 0] \mathbbm{P}(i^{\ast}_{\theta} = i^{\ast}_{\theta}, B_{\theta} \geq 0) + r_{\max}\\ \nonumber
&- \sum_{\theta = 1}^{\infty}  \frac{\mu_{i^{\ast}_{\theta},\theta}}{\eta_{i^{\ast}_{\theta},\theta}}  \mathbbm{E} [ c_{i^{\ast}_{\theta},\theta} | I_{\theta} = i^{\ast}_{\theta}, B_{\theta} \geq c_{\max}] \mathbbm{P}(I_{\theta} = i^{\ast}_{\theta}, B_{\theta} \geq c_{\max}) \\
& \leq \frac{r_{\max}}{c_{\min}} {\Big(}
\mathbbm{E} {\Bigg [} \sum_{\theta=1}^{T^{\ast}(B)} c_{i^{\ast}_{\theta},\theta} {\Bigg ]} - \mathbbm{E} {\Bigg [} \sum_{\theta=1}^{T(B)} c_{i^{\ast}_{\theta},\theta} {\Bigg ]} {\Big)} + r_{\max}
\stackrel{\text{$(b)$}}{\leq} 
\frac{r_{\max}}{c_{\min}} (B - \frac{B}{c_{\max}} c_{\min}) + r_{\max},
% = r_{\max} (\frac{B}{c_{\min}}(1 - \frac{c_{\min}}{c_{\max}}) + 1),
\end{align}
%Double Column Version
%
where $(a)$ holds because of the definition of $B_{\theta}$ and $(b)$ follows from the following facts. The optimal policy stops before it runs out of the budget $B$. Hence, its total paid cost cannot exceed $B$. Moreover, we have $T(B) \geq \frac{B}{c_{\max}}$ and $c_{i,\theta} \geq c_{min}$, $\forall i,\theta$.

For Part 2, we have
%Single column version
\begin{align} \label{eq:inregret_second} \nonumber
    \mathbbm{E}{\Bigg [} \sum_{\theta=1}^{T(B)} r_{i^{\ast}_{\theta},\theta} {\Bigg ]} - \mathbbm{E}{\Bigg [} \sum_{\theta=1}^{T(B)} r_{I_{\theta},\theta} {\Bigg ]}
    = \mathbbm{E}{\Bigg [} \sum_{\theta=1}^{T(B)} (r_{i^{\ast}_{\theta},\theta} - r_{I_{\theta},\theta}) {\Bigg ]}
    & \leq r_{\max} \mathbbm{E}{\Bigg[}\sum_{\theta=1}^{T(B)} \sum_{i = 1}^{S} \mathbbm{1}_{\{I_\theta = i \neq i^{\ast}_{\theta}\}}{\Bigg]} \\
    &= r_{\max} \sum_{i=1}^{S} \mathbbm{E}[\tilde{N}_{T(B)}(i)|T(B)].
\end{align}
By replacing $T(B)$ with $\frac{B}{c_{\min}}$ in (\ref{result2}) and (\ref{Cofsuboptimal2}), substituting the result in (\ref{eq:inregret_second}), and combining (\ref{eq:inregret_first}) and (\ref{eq:inregret_second}) with (\ref{ProblemSolved}), we conclude the proof.

%-------------------------------------------------------------------> Bibliography
\bibliographystyle{IEEEbib}
\bibliography{Main}

\begin{thebibliography}{10}

\bibitem{Josilo17:AGT}
S.~Jo\^{s}ilo and G.~D\'{a}n,
\newblock ``A game theoretic analysis of selfish mobile computation
  offloading,''
\newblock in {\em IEEE INFOCOM 2017 - IEEE Conference on Computer
  Communications}, May 2017, pp. 1--9.

\bibitem{Abbas18:MEC}
N.~Abbas, Y.~Zhang, A.~Taherkordi, and T.~Skeie,
\newblock ``Mobile edge computing: A survey,''
\newblock {\em IEEE Internet of Things Journal}, vol. 5, no. 1, pp. 450--465,
  Feb 2018.

\bibitem{Arif16:ASO}
A.~{Ahmed} and E.~{Ahmed},
\newblock ``A survey on mobile edge computing,''
\newblock in {\em 2016 10th International Conference on Intelligent Systems and
  Control (ISCO)}, Jan 2016, pp. 1--8.

\bibitem{Robbins52:SAS}
Herbert Robbins,
\newblock ``Some aspects of the sequential design of experiments,''
\newblock {\em Bulletin of the American Mathematical Society}, vol. 58, no. 5,
  pp. 527--535, 1952.

\bibitem{Maghsudi17:MAB}
Setareh Maghsudi and Ekram Hossain,
\newblock ``Multi-armed bandits with application to 5g small cells,''
\newblock {\em {IEEE} Wireless Commun.}, vol. 23, no. 3, pp. 64--73, 2016.

\bibitem{Yu17:COF}
S.~Yu, X.~Wang, and R.~Langar,
\newblock ``Computation offloading for mobile edge computing: A deep learning
  approach,''
\newblock in {\em 2017 IEEE 28th Annual International Symposium on Personal,
  Indoor, and Mobile Radio Communications}, 2017, pp. 1--6.

\bibitem{Dinh17:OIM}
T.~Q. Dinh, J.~Tang, Q.~D. La, and T.~Q.~S. Quek,
\newblock ``Offloading in mobile edge computing: Task allocation and
  computational frequency scaling,''
\newblock {\em IEEE Transactions on Communications}, vol. 65, no. 8, pp.
  3571--3584, 2017.

\bibitem{Liu18:OSI}
J.~Liu and Q.~Zhang,
\newblock ``Offloading schemes in mobile edge computing for ultra-reliable low
  latency communications,''
\newblock {\em IEEE Access}, vol. 6, pp. 12825--12837, 2018.

\bibitem{Huang12:ADO}
D.~Huang, P.~Wang, and D.~Niyato,
\newblock ``A dynamic offloading algorithm for mobile computing,''
\newblock {\em IEEE Transactions on Wireless Communications}, vol. 11, no. 6,
  pp. 1991--1995, June 2012.

\bibitem{Munoz15:OOR}
Olga Munoz, Antonio Pascual-Iserte, and Josep Vidal,
\newblock ``Optimization of radio and computational resources for energy
  efficiency in latency-constrained application offloading,''
\newblock {\em IEEE Transactions on Vehicular Technology}, vol. 64, no. 10, pp.
  4738--4755, 2015.

\bibitem{Wang16:MEC}
Y.~Wang, M.~Sheng, X.~Wang, L.~Wang, and J.~Li,
\newblock ``Mobile-edge computing: Partial computation offloading using dynamic
  voltage scaling,''
\newblock {\em IEEE Transactions on Communications}, vol. 64, no. 10, pp.
  4268--4282, 2016.

\bibitem{Roy17:AAC}
Deepsubhra~Guha Roy, Debashis De, Anwesha Mukherjee, and Rajkumar Buyya,
\newblock ``Application-aware cloudlet selection for computation offloading in
  multi-cloudlet environment,''
\newblock {\em The Journal of Supercomputing}, vol. 73, no. 4, pp. 1672--1690,
  2017.

\bibitem{Van18:ADR}
Duc Van~Le and Chen-Khong Tham,
\newblock ``A deep reinforcement learning based offloading scheme in ad-hoc
  mobile clouds,''
\newblock in {\em IEEE INFOCOM 2018-IEEE Conference on Computer Communications
  Workshops}. IEEE, 2018, pp. 760--765.

\bibitem{Ding13:MAB}
Wenkui Ding, Tao Qin, Xu-Dong Zhang, and Tie-Yan Liu,
\newblock ``Multi-armed bandit with budget constraint and variable costs,''
\newblock in {\em AAAI}, 2013.

\bibitem{Xia15:TSB}
Yingce Xia, Haifang Li, Tao Qin, Nenghai Yu, and Tie-Yan Liu,
\newblock ``Thompson sampling for budgeted multi-armed bandits,''
\newblock in {\em IJCAI}, 2015.

\bibitem{Tran12:KBO}
Long Tran-Thanh, Archie Chapman, Alex Rogers, and Nicholas~R Jennings,
\newblock ``Knapsack based optimal policies for budget-limited multi-armed
  bandits,''
\newblock in {\em Proceedings of the Twenty-Sixth AAAI Conference on Artificial
  Intelligence}. AAAI Press, 2012, pp. 1134--1140.

\bibitem{Garivier08:SWUCB}
Aur{\'e}lien {Garivier} and Eric {Moulines},
\newblock ``{On Upper-Confidence Bound Policies for Non-Stationary Bandit
  Problems},''
\newblock {\em arXiv e-prints}, p. arXiv:0805.3415, May 2008.

\bibitem{Maghsudi18:PEP}
S.~Maghsudi and D.~Niyato,
\newblock ``On power-efficient planning in dynamic small cell networks,''
\newblock {\em IEEE Wireless Communications Letters}, vol. 7, no. 3, pp.
  304--307, June 2018.

\bibitem{Xia15:BBP}
Yingce Xia, Wenkui Ding, Xu-Dong Zhang, Nenghai Yu, and Tao Qin,
\newblock ``Budgeted bandit problems with continuous random costs,''
\newblock in {\em ACML}, 2015.

\bibitem{Ganti14:DCA}
Radha~Krishna Ganti and Martin Haenggi,
\newblock ``Dynamic connectivity and path formation time in poisson networks,''
\newblock {\em Wireless Networks}, vol. 20, no. 4, pp. 579--589, May 2014.

\bibitem{Baccelli09:SGA}
Fran{\c c}ois Baccelli and Bartlomiej Blaszczyszyn,
\newblock {\em {Stochastic Geometry and Wireless Networks, Volume I - Theory}},
  vol.~1 of {\em Foundations and Trends in Networking Vol. 3: No 3-4, pp
  249-449},
\newblock {NoW Publishers}, 2009,
\newblock Stochastic Geometry and Wireless Networks, Volume II - Applications;
  see http://hal.inria.fr/inria-00403040.

\bibitem{Haenggi09:SGA}
Martin Haenggi, Jeffrey~G Andrews, Fran{\c{c}}ois Baccelli, Olivier Dousse, and
  Massimo Franceschetti,
\newblock ``Stochastic geometry and random graphs for the analysis and design
  of wireless networks,''
\newblock {\em IEEE Journal on Selected Areas in Communications}, vol. 27, no.
  7, 2009.

\bibitem{Sztrik12:BQT}
J{\'a}nos Sztrik,
\newblock ``Basic queueing theory,''
\newblock 2012.

\bibitem{Ranadheera18:COA}
Shermila Ranadheera, Setareh Maghsudi, and Ekram Hossain,
\newblock ``Computation offloading and activation of mobile edge computing
  servers: A minority game,''
\newblock {\em IEEE Wireless Communications Letters}, 2018.

\bibitem{Takagi84:OTR}
H.~Takagi and L.~Kleinrock,
\newblock ``Optimal transmission ranges for randomly distributed packet radio
  terminals,''
\newblock {\em IEEE Transactions on Communications}, vol. 32, no. 3, pp.
  246--257, 1984.

\bibitem{Contla03:EHC}
Pedro~Acevedo Contla and Milos Stojmenovic,
\newblock ``Estimating hop counts in position based routing schemes for ad hoc
  networks,''
\newblock {\em Telecommunication Systems}, vol. 22, no. 1-4, pp. 109--118,
  2003.

\bibitem{Harb08:ASO}
Shadi~M Harb and Janise Mcnair,
\newblock ``Analytical study of the expected number of hops in wireless ad hoc
  network,''
\newblock in {\em International Conference on Wireless Algorithms, Systems, and
  Applications}. Springer, 2008, pp. 63--71.

\bibitem{Cormen09:ITA}
Thomas~H. Cormen, Charles~E. Leiserson, Ronald~L. Rivest, and Clifford Stein,
\newblock {\em Introduction to Algorithms, Third Edition},
\newblock The MIT Press, 3rd edition, 2009.

\bibitem{Auer02:FTA}
Peter Auer, Nicolo Cesa-Bianchi, and Paul Fischer,
\newblock ``Finite-time analysis of the multiarmed bandit problem,''
\newblock {\em Machine learning}, vol. 47, no. 2-3, pp. 235--256, 2002.

\bibitem{Wroughton13:DBB}
Jacqueline Wroughton and Tarah Cole,
\newblock ``Distinguishing between binomial, hypergeometric and negative
  binomial distributions,''
\newblock {\em Journal of Statistics Education}, vol. 21, no. 1, 2013.

\end{thebibliography}
%-------------------------------------------------------------------
\end{document}